\documentclass[10pt,twocolumn,letterpaper]{article}
\usepackage{cvpr}              

\usepackage{multirow}
\usepackage{amsthm}
\usepackage{svg}
\usepackage{dblfloatfix}
\usepackage{tabularray}
\newcommand{\NM}{MixER\xspace}
\newcommand{\I}{\text{MI\xspace}}
\newcommand{\Bl}[1]{\color{blue}{#1}}
\newcommand{\Gr}[1]{\color{green}{#1}}

\newtheorem{thm}{Theorem}
\newtheorem{defn}{Definition}
\newtheorem{prop}{Proposition}
\definecolor{cvprblue}{rgb}{0.21,0.49,0.74}
\usepackage[pagebackref,breaklinks,colorlinks,citecolor=cvprblue]{hyperref}
\usepackage{accents}
\usepackage{pifont}
\newcommand{\cmark}{\ding{51}}%
\def\paperID{6823} 
\def\confName{CVPR}
\def\confYear{2025}

\title{From Cross-Modal to Mixed-Modal Visible-Infrared Re-Identification}

\author{\textsuperscript{1}Mahdi Alehdaghi,\textsuperscript{1} Rajarshi Bhattacharya, \textsuperscript{2}Pourya Shamsolmoali, \textsuperscript{1}Rafael M. O. Cruz, and \textsuperscript{1}Eric Granger\\
\textsuperscript{1}LIVIA, ILLS, Dept. of Systems Engineering, ETS Montreal, Canada\\
\textsuperscript{2}Dept. of Computer Science, University of York, UK \\
{\tt\small \{mahdi.alehdaghi, rajarshi.bhattacharya\}.1@ens.etsmtl.ca,} 
{\tt\small pshams55@gmail.com,} \\
{\tt\small \{rafael.menelau-cruz, eric.granger\}@etsmtl.ca}
}

\begin{document}
\maketitle
\begin{abstract}
Visible-infrared person re-identification (VI-ReID) aims to match individuals across different camera modalities, a critical task in modern surveillance systems. While current VI-ReID methods focus on cross-modality matching, real-world applications often involve mixed galleries containing both V and I images, where state-of-the-art methods show performance limitations due to large domain shifts and low discrimination across mixed modalities. This is because gallery images from the same modality may have lower domain gaps but correspond to different identities.  This paper introduces a mixed-modal ReID setting, where galleries contain data from both modalities. To address the domain shift among inter-modal and low discrimination capacity in intra-modal matching, we propose the Mixed Modality-Erased and -Related (\NM) method. The \NM learning approach disentangles modality-specific and modality-shared identity information through orthogonal decomposition, modality-confusion, and id-modality-related objectives. MixER enhances feature robustness across modalities, improving cross-modal and mixed-modal settings performance. Our extensive experiments\footnote{Our code is provided in suppl. materials and will be made public.} on the SYSU-MM01, RegDB and LLMC datasets indicate that our approach can provide state-of-the-art results using a single backbone, and showcase the flexibility of our approach in mixed gallery applications.  

\end{abstract} 
\vspace{-0.5cm}
\section{Introduction}
\label{sec:intro}

VI-ReID is crucial in surveillance that allows matching individuals across multiple camera views and lighting conditions. Current VI-ReID methods typically focus on cross-modality matching, where a V or I  image serves as the query against a gallery of images from the other modality (top of Fig.\ref{fig:mix-vs-cross}), like comparing day-time V images to night-time I images. While effective for cross-modality matching, real-world surveillance captures both modalities around the clock, leading to a gallery with mixed V and I images (bottom of Fig.\ref{fig:mix-vs-cross}). This mixed gallery more accurately reflects real-world conditions but introduces a considerable challenge -- same-modality images in the gallery (e.g., V-to-V) have lower domain gaps than cross-modality images (e.g., I-to-V), yet correspond to different identities. Future VI-ReID methods must address these mixed gallery scenarios, balancing cross- and same-modality matching for reliable person ReID in real-world applications.

\begin{figure*}[!t]
\centering
\begin{subfigure}{.12\textwidth}
  \centering
  \frame{\includegraphics[width=\linewidth, height=4.52cm]{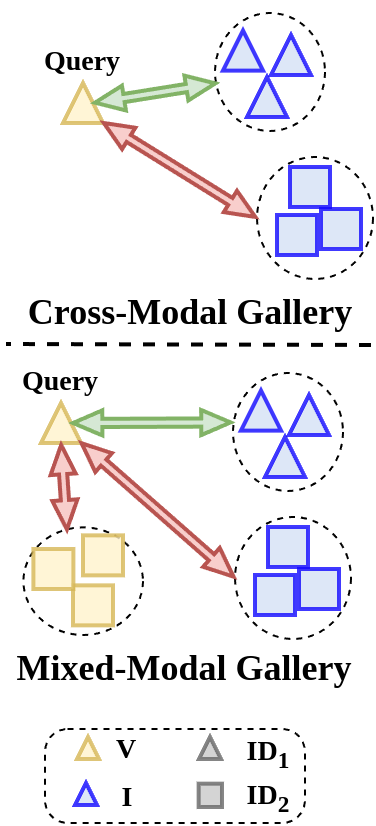}}
  \caption{Cross vs Mixed} 
  \label{fig:mix-vs-cross}
\end{subfigure}%
\hspace{0.05cm}
\begin{subfigure}{.357\textwidth}
  \centering
  \frame{\includegraphics[width=\linewidth]{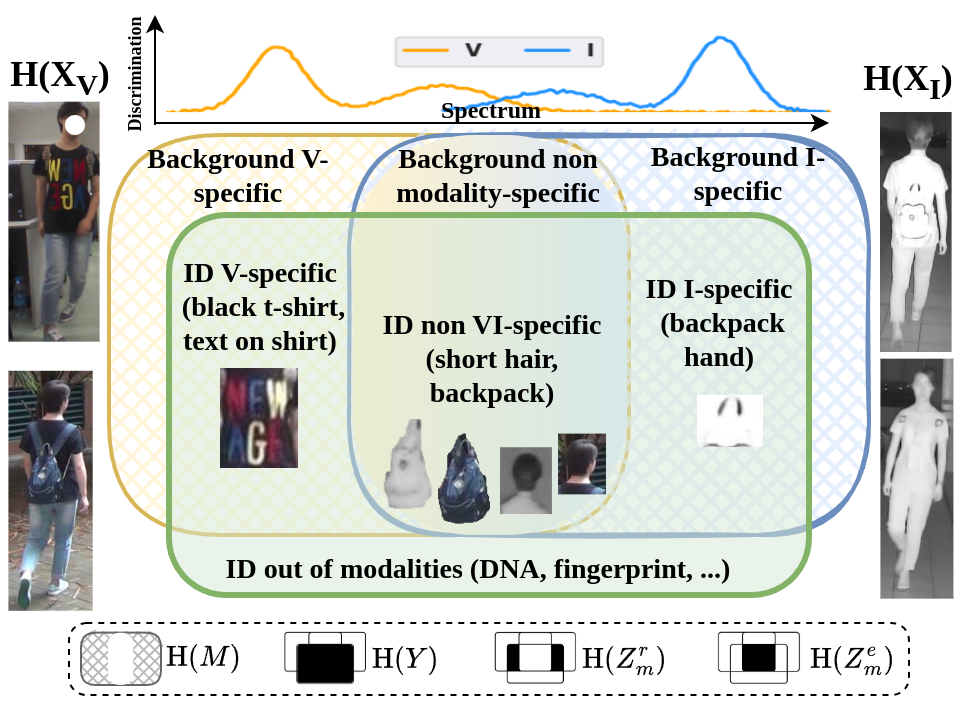}}
  \caption{Venn Diagram}
  \label{fig:venn_intro}
\end{subfigure}%
\hspace{0.05cm}
\begin{subfigure}{.505\textwidth}
  \centering
  \frame{\includegraphics[width=\linewidth]{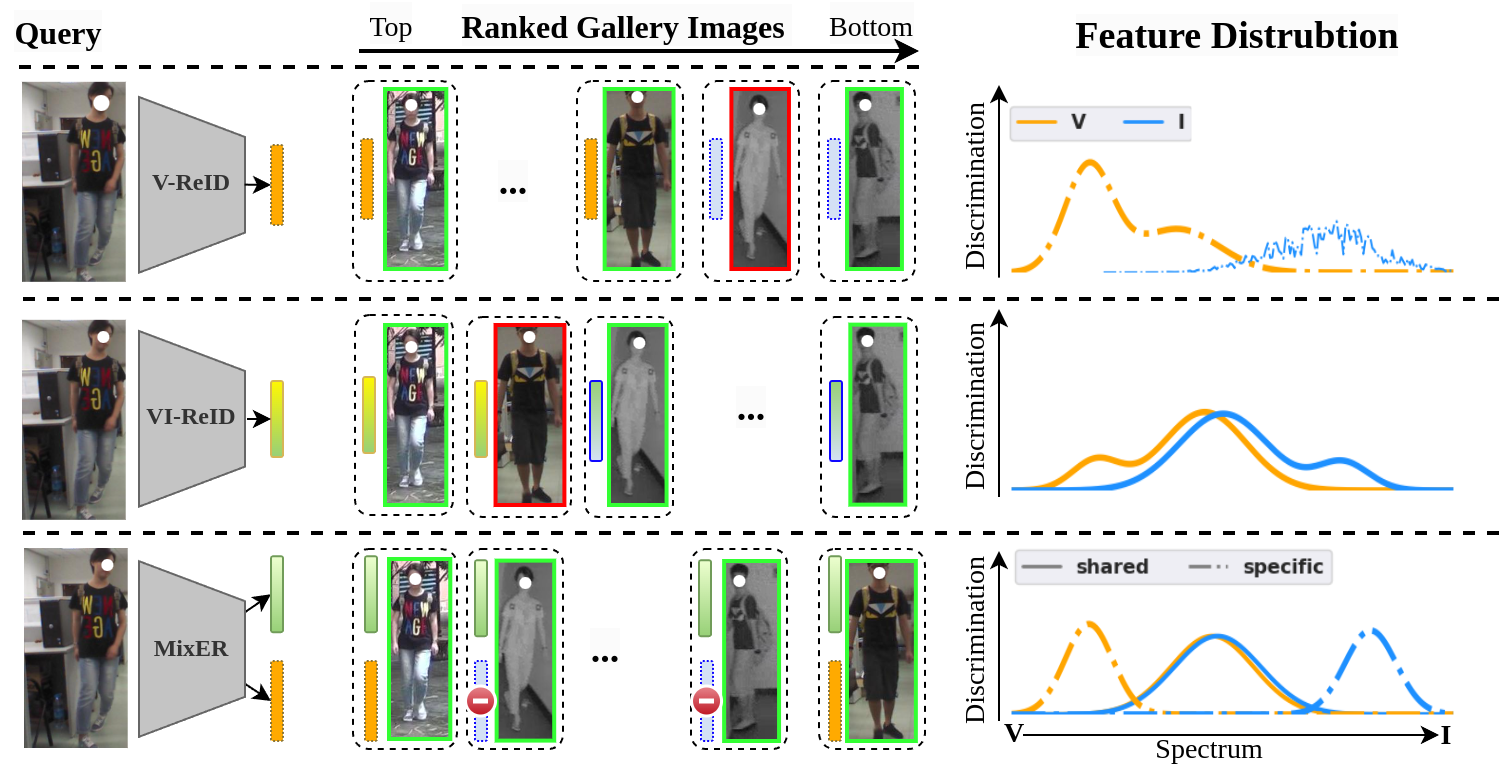}}
  \caption{Mixed-Modal matching}
  \label{fig:mixed-match}
\end{subfigure}%
\vspace{-0.3cm}
\caption{(a) VI-ReID of a query image matched against a cross-modal and mixed-modal gallery. (b) With VI-ReID methods, V, I, and ID information reveal modality-specific and modality-shared ID features. (c) Mixed-modal approaches leverage these features for matching, while uni-modal methods are limited to intra-modality matching due to a large modality gap, and cross-modal methods enable inter-modal matching by extracting shared features. Our MixER method disentangles these features to reduce the gap in shared features and enhance the discrimination in modality-specific features.}

\label{fig:mixed-gallery}
\vspace{-0.6cm}
\end{figure*}



Mixed-modal ReID settings can be interpreted in various ways. We identify four distinct interpretations w.r.t. the identity and camera linked to query and gallery images. In each setting, specific images are removed to measure the robustness of the matching model. 
VI-ReID methods are generally expected to perform well in mixed-modality settings. 
However, in this paper, two settings are shown to present particular challenges where state-of-the-art uni-modal and cross-modal ReID methods struggle to achieve a high level of performance.  A cost-effective VI-ReID method in these settings should rely on a single backbone model to extract discriminative information, and then utilize the appropriate features based on intra-modality or inter-modality matching. Fig.\ref{fig:venn_intro} illustrates the information of the two modalities w.r.t. identity space. The discriminative attributes can be divided into modality-specific and modality-shared information. V-ReID models focus on extracting modality-specific features, while V-ReID methods attempt to extract modality-shared features.

Given the substantial discrepancy between I and V modality images, uni-modal models typically struggle to remain discriminative across different modalities (as illustrated in the first row of Fig.\ref{fig:mixed-match}). 
To address this challenge, VI-ReID methods often focus on minimizing this discrepancy between their extracted feature representations. For example, GAN methods \cite{Wang_2019_ICCV_AlignGAN,kniaz2018thermalgan} are used to translate the modalities and shared-backbones~\cite{all-survey,cui2024dma,park2021learning, randLUPI, IDKL, shape-Erase23, LLCM} are used to map images into a modality-invariant feature space, creating a shared representation that minimizes modality discrepancies. Despite the significant improvements in recent years, the features extracted by these methods often contain some modality-specific information, which weakens the effectiveness of inter-modal matching and limits the minimization of modality discrepancies in the shared feature space. 
A limitation of current VI-ReID methods in mixed matching tasks is their lack of attention to identity features that exist only in one modality, as shown in the second row of Fig.\ref{fig:mixed-match}. Some attributes, such as shirt patterns that are only relevant in the V modality or material textures unique to I, are modality-specific and, therefore, unavailable in the other modality (see Fig.\ref{fig:venn_intro}). For accurate cross-modal matching, this modality-specific information should be erased from identity representations, only allowing access to modality-shared information across both V and I. However, effectively separating modality-specific from modality-invariant is challenging, as it requires effective unsupervised disentanglement to identify them.

To address the limitations of VI-ReID in mixed-modal scenarios, this paper introduces a \textbf{Mix}ed Modality-\textbf{E}rased and Modality-\textbf{R}elated (MixER) ID-discriminative feature learning approach.
By isolating discriminative modality-specific features within one subspace, our approach promotes modality-erased features to capture robust discriminative semantic concepts across modality variations within the other subspace. Implicit Discriminative Knowledge Learning (IDKL) \cite{IDKL} uses modality-specific features to enhance modality-shared ones through knowledge distillation and implicit similarity. In contrast, our approach enforces an orthogonal complement structure. This constraint assures that modality-related features remain independent of shared features. This allows discovering a diverse embedding space to ensure both modality-specific and -shared are effectively used. Building on these assumptions and inspired by \cite{shape-Erase23}, we formulate our objectives from a mutual information perspective, demonstrating that joint learning of modality-erased and modality-related features optimizes mutual information with identity, producing effective feature representations for mixed-modal matching. 

As illustrated in Fig.\ref{fig:mixed-match}, \NM leverages modality-erased features for inter-modality matching while refining intra-modality matching by mixing them with modality-related features via ID-modality-aware, modality-confusion, and feature-fusion losses. Furthermore, our framework can be integrated into state-of-the-art approaches to enhance their performance in mixed-modal and cross-modal matching settings.  

The main contributions of this paper are summarized as follows.
    \noindent \textbf{(1)} We motivate and formalize a new evaluation Mixed-Modal ReID setting, where galleries may have data from both I and V modalities. To the best of our knowledge, we are the first to propose and benchmark such a setting for evaluating VI-ReID methods. 
    \noindent \textbf{(2)} A mixed modality-erased and -related (\NM) feature learning paradigm is introduced for enhancing mixed- and cross-modal VI-ReID on a single feature embedding backbone, separating modality-erased features from modality-related ones through orthogonal decomposition and gradient reversal. Modality-erased features capture modality-shared discriminative semantics, while modality-related features are designed to extract additional discriminative attributes unique to each modality, thereby improving the diversity of learned representations. 
    \noindent \textbf{(3)} Our extensive experiments on the challenging SYSU-MM01, RegDB, and LLCM datasets indicate that our MixER method can outperform state-of-the-art VI-ReID approaches in cross-modal and mixed-modal settings. They also show the flexible integration of MixER to enhance state-of-the-art models. 
\section{Related Work}
\vspace{-0.2cm}
\label{sec:related}

\noindent \textbf{(a) Cross-modal Person Re-Identification.}
Person ReID is the task of identifying distinctive characteristics in sample query images of individuals and matching these characteristics within a larger gallery of images~\cite{chen2017person, chen2018improving, ye2021deep, zheng2016person, zheng2020dual, zhou2023adaptive}.
To address the person ReID task in low-light conditions, V-I ReID has emerged as an important area of research.
Cross-modal VI ReID techniques primarily focus on the extraction of global representations \cite{Wang_2019_ICCV_AlignGAN, kniaz2018thermalgan, randLUPI, all-survey, Bi-Di_Center-Constrained, hetero-center, park2021learning} or local part-based representations through techniques like horizontal striping~\cite{DDAG,cmSSFT} or attention mechanisms \cite{part1,wu2021Nuances,alehdaghi2024bidirectional}. While these representations sufficiently discriminate between identities, they do not adequately address modality-specific discrepancies, leading to less robust modality-free features. To address these issues, recent methods focus on extracting modality-invariant features by disentangling them from modality-specific information \cite{HI-CMD, paired-images2, cmSSFT, zhu2023information, zheng2019camera, lu2024disentangling, shape-Erase23}. There exist some generative techniques \cite{HI-CMD, paired-images2, cmSSFT, alehdaghi2023adaptive} that aim to separate content information from modality-related style attributes. However, generating synthetic images can result in the loss of identity-related information. Some approaches, such as \cite{zhu2023information, zheng2019camera}, reconstruct features while enforcing orthogonality between modality-specific and modality-invariant information, but challenges remain in ensuring erased features do not retain modality-specific elements. Shape information is leveraged in \cite{lu2024disentangling,shape-Erase23} to disentangle body-shape information from modality-aware identity information to make the extracted feature not dependent on shape. We derive some of our insights from their interesting work.

\noindent \textbf{(b) Multi-modal Learning.}
Different data modalities offer complementary features for representation, an area explored in multi-modal learning \cite{liang2022foundations}. Recent transformer-based methods \cite{gabeur2020multi, kim2021vilt, li2021align} fuse multi-modal inputs directly as tokens rather than extracting separate modality-specific features. Real-world applications, however, often lack complete modality availability, presenting challenges addressed by approaches that handle missing modalities \cite{wang2023multi, pan2021disease, cai2018deep, lee2023multimodal}. The ImageBind model \cite{girdhar2023imagebind} exemplifies large-scale multi-modal unification, mapping diverse modalities into a shared representation space. In ReID, \cite{li2024all} presents a unified model across RGB, infrared, sketch, and text modalities, achieving strong cross-domain performance. While related to VI-ReID, this multi-modal approach does not fully address the unique challenges of visible-infrared matching. Instruct-ReID \cite{he2024instruct} further unifies ReID tasks by leveraging a multi-modal backbone for text and image prompts, achieving state-of-the-art performance across various ReID settings.

\noindent \textbf{(c) Critical Analysis.} 
While significant focus has been placed on cross-modal V-I person ReID and its associated challenges, the mixed-modality query-gallery scenario remains largely unaddressed. Despite achieving SOTA results in standard settings, existing VI-ReID methods did not report their performance in mixed-modality settings. Furthermore, while some methods \cite{IDKL, zhu2023information} attempt to disentangle modality-specific features to enhance the modality-shared component in cross-modal contexts, we show that directly using modality-specific information does not aid inter-modal matching, where such information is unavailable and must instead be erased from shared features. However, these modality-specific features prove beneficial for intra-modal matching within mixed-modality settings. Our findings are further supported by mutual information analysis and experimental results across proposed mixed-modal scenarios.



\section{The Proposed \NM Method}
\label{sec:proposed}
To address the limitation of ReID methods in mixed-modal settings, we propose the MixER learning paradigm to combine ID-discriminative modality-erased and modality-related features. This enhances the discriminative capacity of features by utilizing modality-specific information. Fig.\ref{fig:method} provides an overview of the proposed \NM learning approach. It relies on a shared backbone with three sub-modules to extract independent modality-related and -erased features by applying the orthogonal decomposition, modality-confusion, and losses related to ID modality.  

A multimodal dataset for VI-ReID is composed of visible $\mathcal{V}=\{x^{(j)}_v, y_v^{(j)}\}_{j=1}^{N_v}$ and infrared $\mathcal{I}=\{x^{(j)}_i,y_i^{(j)}\}_{j=1}^{N_i}$ sets of images from $C_y$ distinct individuals, with their ID labels. Our proposed Mixed VI-ReID system seeks to match images captured from V and I cameras by using one deep backbone model that encodes modality-invariant person embeddings, denoted by $\mathbf{z}_v$ and $\mathbf{z}_i$. Given query images (V or I), the objective is to retrieve images with the same identity over the gallery set containing both V and I modalities, by computing and sorting the distance value $D(.,.)$ for each gallery image: 
\vspace{-0.2cm}
\begin{equation}
    \label{eq:problem}
    D(\mathbf{z}_m^{(j)}, \mathbf{z}_{m'}^{(p)}) < D(\mathbf{z}_m^{(j)}, \mathbf{z}_{m''}^{(n)}), 
    \vspace{-0.2cm}
\end{equation}
where $y_m^{(j)}=y_{m'}^{(p)} \neq y_{m''}^{(n)}$ , $m,m',m''$ are modalities that could be $v$ or $i$ independently, and superscripts $p$ and $n$ indicating indices of positive and negative samples, respectively. To learn these features, we decompose them into two independent components, each meeting specific constraints suitable for inter-modal and intra-modal matching. This is achieved by maximizing the mutual information (MI) between these features and the ID labels.


\subsection{Mutual Information Analysis}
To extract ID-discriminative features from images, the model needs to maximize the MI between these extracted features\footnote{\noindent{Uppercase is used as random variables and lowercase for samples.}}, $Z_{m}$, and label spaces, $Y$:
\vspace{-0.25cm}
\begin{equation}
\vspace{-0.25cm}
\label{eq:mi_main}
    \max_{Z_{m}}  \quad \I(Z_{m};Y), 
\end{equation}
where $Y$ represents the identity of the individuals in the input images. To ensure that the learned features are effective in mix-modality scenarios, they are decomposed into two independent components:  \textbf{(a) modality-erased} ($Z^e_{m}$), which should not contain any modality information to be proper for inter-modal matching (e.g., short-hair attributes in Fig.\ref{fig:venn_intro}), and \textbf{(b) modality-related} ($Z^r_{m}$) component to refine the modality-erased for improving the intra-modal matching. This component should contain ID information that is also relevant to the modality (e.g., text on a t-shirt in Fig.\ref{fig:venn_intro}). To have two independent components, the MI between them must be zero. Thus, the optimization becomes:
\begin{equation}
\label{eq:mi_dual}
\begin{aligned}
     \max_{Z^e_{m},Z^r_{m}}\quad \{ \I(Z^e_{m},Z^r_{m};Y) \}\; 
\textrm{s.t.}& \;\; \I(Z^e_{m}; Z^r_{m})=0, \\
\; \I(Z^e_{m}; M)=0 \;\; \textrm{and}& \;\; \I(Z^r_{m}; Y|M)=0,
\end{aligned}
\end{equation}
where $M$ is the modality label space. Constraint $\I(Z^e_{m}; M)$=0 ensures that modality information is erased from $Z^e_{m}$ and $\I(Z^r_{m};Y|M)$=0 ensures that $Z^r_{m}$ does not contain ID-aware information, which disregards the modality label. Also, $Z^e_{m}$ and $ Z^r_{m}$ should be independent.
\begin{thm} \label{thm:1m} If $Z^e_{m}$ and $Z^r_{m}$ are independent, then $\I(Z^e_{m},Z^r_{m};Y)=\I(Z^e_{m};Y) + \I(Z^r_{m};Y)$.
\end{thm}
\begin{proof}
Can be found in the supplementary materials.  
\end{proof}
\noindent By incorporating $M$ into Eq. (\ref{eq:mi_dual}) and using Theorem \ref{thm:1m}: 
\begin{equation}
\begin{aligned}
    \I(Z^e_{m},Z^r_{m};Y)\!=& \I(Z^e_{m};Y) + \I(Z^r_{m};Y) \\
    =& \I(Z^e_{m};Y|M) + \I(Z^e_{m};Y;M) \\
    +& \I(Z^r_{m};Y|M) + \I(Z^r_{m};Y;M).
\end{aligned}
\end{equation}
Since the \I(.;.) is non-negative, $\I(Z^e_{m};Y;M)$=0, and $\I(Z^r_{m};Y|M)$=0, 
Eq. (\ref{eq:mi_dual}) can be formulated as the maximization of the following Lagrangian:
\begin{equation}
\begin{aligned}
    \max_{Z^e_{m},Z^r_{m}} \{\!& \overbrace{ \I(Z^e_{m};Y|M)\!- \!\lambda_1\I(Z^e_{m}; M)}^{\textit{Modality-Erased Learning} } + \\
    \!&\underbrace{\I(Z^r_{m};Y;M)}_{\textit{Modality-Related Learning}}  - \underbrace{\lambda_2 \I(Z^e_{m}; Z^r_{m})}_{\text{Orthogonal Feature Learning}}\}.
\end{aligned}
\end{equation}

\begin{figure}[t]
\centering
\includegraphics[width=0.99\columnwidth]{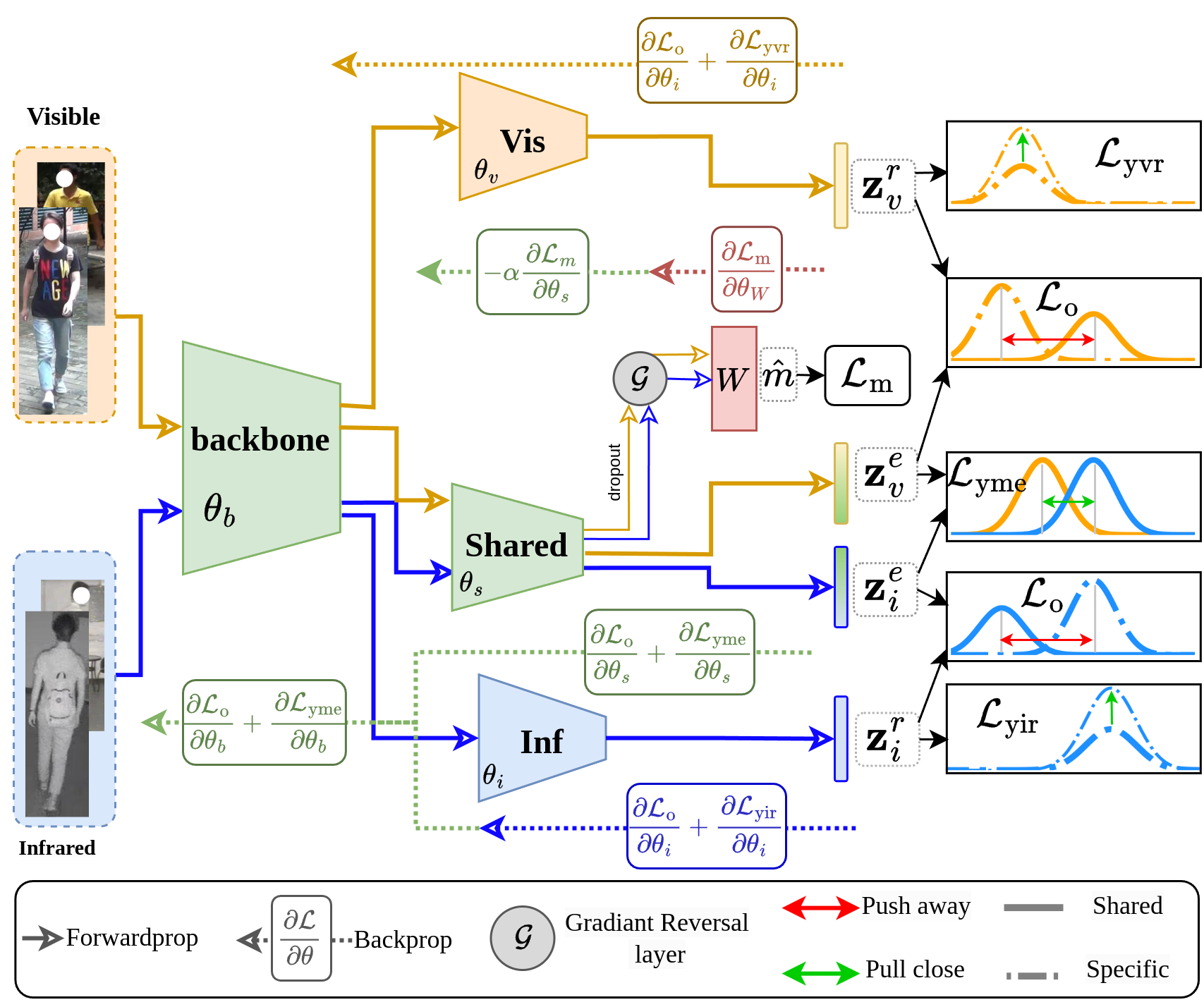} 
\vspace{-0.3cm}
\caption{The overall architecture of our proposed MixER method. It extracts two independent ID-discriminative feature vectors by orthogonal decomposition, modality-confusion, and modality-aware losses for learning modality-erased and modality-related feature representation. }
\label{fig:method}
\vspace{-0.3cm}
\end{figure}

\subsection{Modality Erased and Related Learning}

\noindent {\textbf{(a) Orthogonal Feature Decomposition.}}
To decompose the extracted features into modality-erased and -related feature vectors,  
two modality-specific and one shared sub-modules are proposed to project the $\mathbf{z}_m$ to them as:
\begin{equation}
    \mathbf{z}^e_m = \mathcal{F}_s(\mathbf{z}_m), \ \ \mathbf{z}^r_m =\mathcal{F}_m(\mathbf{z}_m),
\end{equation}
where $\mathbf{z}_m=\mathcal{F}_b(x_m)$ is extracted by a shared backbone.
One reason behind using specific sub-modules for modality-related features is to prevent them from sharing information through model parameters. 

\noindent {\textit{Minimizing} $\I(Z^e_{m}; Z^r_{m})$}: When minimizing the MI between $Z^r_m$ and $Z^e_m$ they must be independent to avoid affecting each other through the learning. Since the MI estimation is complex and time-consuming~\cite{MINE,shape-Erase23}, we estimate the independence constraint as minimizing the orthogonal loss:
\begin{equation}
    \mathcal{L}_\text{o} =  \mathbb{E}_{(\mathbf{z}^r_m,\mathbf{z}^e_m)\sim (Z^r_m,Z^e_m)} \dfrac {\mathbf{z}^r_m \cdot \mathbf{z}^e_m} {\left\| \mathbf{z}^r_m\right\| _{2}\left\| \mathbf{z}^e_m\right\| _{2}} 
\end{equation}

\noindent {\textbf{(b) Learning Modality Erased Features.}}
The goal of modality-erased learning is extracting features, $\mathbf{z}^r_m$, that discriminate the ID without using modality information to make them appropriate for cross-modal matching when the modality-specific information is absent.  

\noindent {\textit{Maximizing} $\I(Z^e_{m};Y|M)$}: given the MI attributes, we expand it as (The proof can be found in Sec. \ref{sec:proofs} of suppl. materials.): 
\begin{equation}
\I(Z^e_{m};\!Y\!|\!M\!)\!=\!\I(Z^e_{m}\!;Y\!)-\I(Z^e_{m};Y;M\!)\!=\!\I(\!Z^e_{m};Y\!).
\label{eq:ym}
\end{equation}
To maximize Eq. (\ref{eq:ym}), the classification loss $\mathcal{L}_\text{id}$ must be minimized (see the property \ref{prop:cross} in supply. materials):
\begin{equation}
    \mathcal{L}_\text{meid}(Z_m^e,Y) = \mathbb{E}_{(\mathbf{z}^e_m,y_m)\sim (Z^e_m,Y)} \mathcal{L}_\text{ce}(\mathbf{z}^e_m, y_m),
\end{equation}
where $\mathcal{L}_\text{ce}$ is cross-entropy loss. To obtain distinct features for each person, the center-cluster loss \cite{wu2021Nuances} and $\mathcal{L}_\text{cc}$ are minimized. Our modality-erased identity loss is:
\begin{equation}
    \mathcal{L}_\text{yme} = \mathcal{L}_\text{meid}(Z_m^e,Y) + \mathcal{L}_\text{cc}(Z_m^e,Y),
\end{equation}
\noindent {\textit{Minimizing} $\I(Z^e_{m}; M)$}: Modality-erased features should not distinguish the origin modality of input samples. Therefore, an adversarial modality classifier based on the Gradient Reversed Layer (GRL)\cite{GRL} is used to propagate the reverse gradient onto model parameters. The modality-confusion loss for this objective is:
\vspace{-0.2cm}
\begin{equation}
    \mathcal{L}_\text{m} = \mathbb{E}_{(\mathbf{z}^e_m,m)\sim (Z^e_m,M)} \mathcal{L}_\text{ce}(\mathcal{G}(W^T\mathbf{z}^e_m), m),
    \vspace{-0.2cm}
\end{equation}
where $\mathcal{G}$ is the GRL and $W \in \mathbb{R}^{d\times 2}$ is a linear layer.

\noindent {\textbf{(c) Learning Modality Related Features.}}
To learn features that leverage modality-related ID-discriminating information simultaneously, a new doubled label space is proposed to account for identity alongside the modality by separating the same person in each modality: 
\vspace{-0.2cm}
\begin{equation}
y'_{m} \sim Y'  = \begin{cases}
2y_m &m = \text{v}\\
2y_m + 1 &m = \text{i},
\vspace{-0.2cm}
\end{cases} 
\label{eq:disc_label}
\end{equation}  
and minimizing the cross-entropy loss between $\mathbf{z}^r_m$ and $y'$:
\begin{equation}
    \mathcal{L}_\text{mrid}(Z_m^r,Y')= \mathbb{E}_{(\mathbf{z}^r_m,y'_m)\sim (Z^r_m,Y')} \mathcal{L}_\text{ce}(\mathbf{z}^r_m, y'_m).
\end{equation}
The modality-aware loss function for $Z_m^r$ is :
\begin{equation}
    \mathcal{L}_\text{ymr} = \mathcal{L}_\text{mrid}(Z_m^r,Y') + \mathcal{L}_\text{cc}(Z_m^r,Y'),
\end{equation}
where the $\mathcal{L}_\text{cc}$ is the center-cluster loss \cite{wu2021Nuances}.

\noindent {\textbf{(d) Mixed-Modal Matching Fusion.}}
To balance modality-erased and modality-related features for the matching process at the inference time, we propose a mixed cross-modal triplet loss to avoid having one component be dominant or useless. The feature fusion loss is:
\begin{equation}
\begin{aligned}
    \mathcal{L}_\text{f}\! &=\!\max\! \{ D(\mathbf{z}^{f,(j)}_m,\! \mathbf{z}^{f,(p)}_m)\! -\! D(\mathbf{z}^{e,(j)}_m, \mathbf{z}^{e,(n)}_{\Tilde{m}})\! +\! \alpha,0 \} \\
    &+\! \max \{ D(\mathbf{z}^{e,(j)}_m, \mathbf{z}^{e,(p)}_{\Tilde{m}})\! -\! D(\mathbf{z}^{f,(j)}_m,\! \mathbf{z}^{f,(n)}_m)\!  +\! \alpha,0 \}
\end{aligned}
\end{equation}
where $D(.,.)$ is the distance between two embeddings, $\Tilde{m} \neq m$, $n$ and $p$ are positive and negative samples to the $j$ instance. $\mathbf{z}^f_m$ is formed by concatenating $\mathbf{z}^e_m$ and $\mathbf{z}^r_m$.

\noindent {\textbf{(e) Overall Training.}}
We jointly optimize the network in an end-to-end manner by using the overall loss:
\vspace{-0.2cm}
\begin{equation}
\label{eq:all_losses}
    \mathcal{L} = \mathcal{L}_{\text{yme}} + \mathcal{L}_{\text{ymr}}+ \lambda_{\text{m}} \mathcal{L}_{\text{m}} + \lambda_{\text{o}} \mathcal{L}_{\text{o}} + \lambda_{\text{f}} \mathcal{L}_{\text{f}},
    \vspace{-0.1cm}
\end{equation}
and $\lambda_{\text{m}}$, $\lambda_{\text{o}}$ and $\lambda_{\text{f}}$ are hyperparameters for weighting losses. 

\subsection{Inference }
During inference, \NM performs mixed-modal matching using a single shared backbone and three light heads rather than using three different backbones. For a given input image $x_m$, it extracts $\mathbf{z}_m^r$, $\mathbf{z}_m^e$ and combines them into $\mathbf{z}_m^f$. Matching scores for a query input image are computed using cosine similarity. Features fused by simple concatenation are used for images within the same modality, while cross-modal matching uses only modality-erased features.
\section{Experimental Methodology and Results}


\begin{figure*}[!t]
    \centering
    \includegraphics[width=\linewidth]{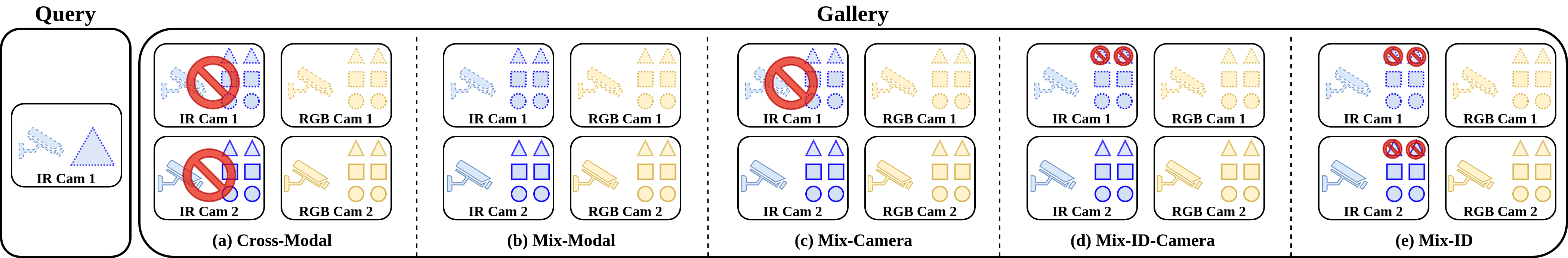}
    \vspace{-0.6cm}
    \caption{Different settings for forming a gallery based on modality, camera, and query identity. The gallery set images are (a) all images from another modality, (b) all images from both modalities, and all images except ones with the (c) same camera, (d) same camera and same identity, and (e) same identity in the same modality as the query. }
    \label{fig:mixed-gallery}
    \vspace{-0.5cm}
\end{figure*}

\noindent\textbf{Datasets.} 
Research on cross-modal V-I ReID has extensively used the SYSU-MM01 \cite{SYSU}, RegDB \cite{regDB} datasets, and recently published LLCM \cite{LLCM} datasets. SYSU-MM01 is a large dataset containing more than 22K V and 11K I images of 491 individuals captured with 4 RGB and 2 NIR cameras. 
RegDB contains 4,120 co-located V-I images of 412 individuals. Randomly divide the dataset into two sets of the same size for training and testing. The LLCM dataset consists of a large, low-light, cross-modality dataset comprising 1064 identities and is divided into training and testing sets at a 2:1 ratio.

\begin{table*}[!phb]
\centering
\resizebox{0.87\textwidth}{!}{%
\begin{tblr}{
  colspec={|l|c|cc|cc|cc|cc||cc|cc||cc|cc|},
  cell{1}{3} = {c=8}{c},
  cell{1}{11} = {c=4}{c},
  cell{1}{15} = {c=4}{c},
  cell{1}{1} = {r=3}{},
  cell{1}{2} = {r=3}{},
  cell{2}{3} = {c=2}{},
  cell{2}{5} = {c=2}{},
  cell{2}{7} = {c=2}{},
  cell{2}{9} = {c=2}{},
  cell{2}{11} = {c=2}{},
  cell{2}{13} = {c=2}{},
  cell{2}{15} = {c=2}{},
  cell{2}{17} = {c=2}{},
  hline{1-2,4,13} = {-}{},
  hline{3} = {2-18}{},
  colsep=3.5pt,
  stretch=0
}
\textbf{Method} & Venue  & Mixed-Modal    &                &                  &                &                     &                &                 &                & Cross-Modal  &              &                 &              & Uni-Modal                  &              &              &              \\
\textbf{Method} & Venue & \textbf{Mix}   &                & \textbf{Mix-Cam} &                & \textbf{Mix-Cam-ID} &                & \textbf{Mix-ID} &                & \textbf{All} &              & \textbf{Indoor} &              & \textbf{I$\rightarrow$I} &              & \textbf{V$\rightarrow$V} &              \\
      &          & \textbf{R1}    & \textbf{mAP}   & \textbf{R1}      & \textbf{mAP}   & \textbf{R1}         & \textbf{mAP}   & \textbf{R1}     & \textbf{mAP}   & \textbf{R1}  & \textbf{mAP} & \textbf{R1}     & \textbf{mAP} & \textbf{R1}                & \textbf{mAP} & \textbf{R1}  & \textbf{mAP} \\
DDAG\cite{DDAG} & ECCV20 & 91.91 & 57.60 & 80.81 & 53.54 & 70.41 & 45.14 & 23.92 & 29.09 & 53.29 & 50.61 & 61.02 &  67.98 & 80.15 & 86.80 & 97.89 & 91.90\\
MPANet \cite{wu2021Nuances}  & CVPR21 & 95.46 & 72.82 & 89.45 & 70.34 & 81.95 & 63.81 & 49.06 & 51.08 & 70.58 & 68.24 & 76.54 & 80.95 & 88.53 & 92.85 & 98.31 & 94.08\\
DEEN\cite{LLCM} + Flip &  CVPR23 & - & - & - & - & - & - & - & - & 74.7 & 71.8 & 80.3 & 83.3 & - & - & - &- \\
DEEN\cite{LLCM}  &  CVPR23& 94.07 & 74.06 & 87.50 & 72.09 & 79.50 & 65.66 & 56.77 & 55.30 & 72.55 & 68.59 & 84.21  & 82.14 & 86.14 & 90.98 & 96.83 & 89.93\\
SGEIL\cite{shape-Erase23} & CVPR23 & 94.82 & 71.03 & 89.22 & 70.16 & 79.33 & 61.34 & 46.60 & 48.37 & 73.34 & 67.44 & 84.09 & 81.65 & 86.52 & 91.36 & 97.29 & 90.24\\
SAAI\cite{SAAI} +AIM & ICCV23 & - & - & - & - & - & - & - & - & 75.90 & 77.03 & 83.20 & 88.01 & - & - & - & - \\
SAAI\cite{SAAI} & ICCV23 & 96.01 & 74.59 & 90.63 & 72.51 & 84.30 & 65.94 & 52.49 & 53.30 & 73.87 & 69.71 & 84.19 & 82.59 & 89.29 & 93.06 & 98.24 & 93.46\\
IDKL\cite{IDKL} + KR & CVPR24 & - & - & - & - & - & - & - & - & 81.42 & 79.85 & 87.14 & 89.37 &  
- & - & - & -\\
IDKL\cite{IDKL}  &CVPR24 & 95.27 & 72.81 & 90.15 & 71.87 & 80.73 & 63.54 & 47.34 & 50.86 & 72.05 & 69.67 & 84.54 & 83.76 & 89.55 & 93.18 & 98.55 & 95.13\\
\textbf{\NM}(R) & & 96.22 & 38.94 & 88.32 & 25.37 & 82.58 & 23.65 & 0 & 0.99 & 1.19 & 3.61 & 1.94 & 5.2 &90.02 & 93.93 & 98.63 & 94.81\\
\textbf{\NM}(E) & & 96.46 & 75.21 & 90.24 & 72.89 & 84.55 & 66.02 & 52.89 & 54.86 & 75.04 & 71.22 & 85.82 & 84.06 &90.18 & 93.65 & 97.95 & 93.38\\

\textbf{\NM} & Ours & \textbf{96.63} & \textbf{79.64} & \textbf{91.77} & \textbf{76.35} & \textbf{87.56} & \textbf{72.70} & \textbf{65.14} & \textbf{62.76} & 75.04 & 71.22 & 85.82 & 84.06 & 93.06 & 95.82 & 99.14 & 95.30 \\
\textbf{\NM} +AIM  & & - & - & - & - & - & - & - & - & 76.12 & 75.45 & 85.30 & 85.54 & - & - & - & - \\
\textbf{\NM} +KR &  & - & - & - & - & - & - & - & - & \textbf{85.96} & \textbf{83.54} & \textbf{92.48} & \textbf{91.49} & - & - & - & - \\ \hline
\end{tblr}
}
\vspace{-0.3cm}
\caption{Performance of SOTA VI-ReID techniques in mixed, cross, and uni-modal settings on the SYSU-MM01 dataset. AIM (Affinity Inference) and KR (k-reciprocal) are different re-ranking processes that are applied on SAAI\cite{SAAI} and IDKL\cite{IDKL} original GitHub projects.}
\label{tab:SYSU}
\end{table*}

\noindent\textbf{Mixed-Modal Evaluations.}
Since mixed-modal evaluation has not been performed on V-I datasets, we introduce multiple settings of such assessment based on real-world scenarios where the query and gallery images can be either I or V. Unlike cross-modality settings, where the query and gallery sets strictly contain images of opposite modalities, the mixed-modal setting may contain images from both modalities in the query and gallery sets. We structured the breadth of mixed-modality settings into 4 cases as illustrated in Fig.\ref{fig:mixed-gallery}, each one defined by the exclusion of images based on the query identity and camera:

\noindent{1.} \textbf{Mix}: No images are excluded. Images of all individuals and cameras are included, regardless of the query camera type. 

\noindent{2.} \textbf{Mix-Camera}: Excludes images taken by the same camera as the query.

\noindent{3.} \textbf{Mix-Camera-ID}: Excludes images taken by the same camera as the query and belonging to the query person. 

\noindent{4.} \textbf{Mix-ID}: Excludes all images of the query person captured by cameras of the same modality as the query. 

Since ReID methods have not been evaluated in these settings, we retrained several open-sourced state-of-the-art (SOTA) VI-ReID methods to assess their performance across these configurations. Rank-1 (R1) accuracy and mean average precision (mAP) were measured for each setting in line with the dataset evaluation criteria. Note that for the RegDB and LLCM datasets, there is only one camera per modality, therefore the "Mix-Camera" and "Mix-Camera-Identity" settings are not applicable.

\noindent\textbf{Implementation Details.}
To extract modality-erased features, we use the SAAI model\cite{SAAI} without prototype learning as a baseline with ResNet50 \cite{resnet} as the backbone. 
For modality-related learning modules, \texttt{layer4} is cloned from the backbone for each modality. Each image input is resized to 288 by 144, then cropped and erased randomly, and filled with zero padding or mean pixels. We used an ADAM optimizer \cite{kingma2014adam} with a linear warm-up strategy for the optimization process. Each training batch contains 8 V and 8 I images from 10 randomly selected identities.  The model was trained for 180 epochs, following \cite{SAAI}, the initial learning rate is set to 0.0004 and decreased by factors of 0.1 and 0.01 at 80 and 120 epochs, respectively. 

\noindent\textbf{Benchmark methods.} To evaluate the effectiveness according to different settings, 
we benchmarked our proposed \NM method against several open-source SOTA VI-ReID methods: DDAG \cite{DDAG}, MPANet \cite{wu2021Nuances}, DEEN \cite{LLCM}, SGEIL \cite{shape-Erase23}, SAAI \cite{SAAI}, and IDKL \cite{IDKL}. Each method was retrained from scratch using the parameters given by the authors. The implementation details of these methods are described in the suppl. materials.

\subsection{Comparison with State-of-the-Art Methods}

\begin{table*}[!ht]
  \small
  \centering
  \begin{subtable}{.48\textwidth}
  \centering
    
\resizebox{\textwidth}{!}{
\begin{tblr}{
  colspec={|l|cc|cc||cc|cc|},
  cell{1}{1} = {r=3}{},
  cell{1}{2} = {c=4}{c},
  cell{1}{6} = {c=4}{c},
  cell{2}{2} = {c=2}{},
  cell{2}{4} = {c=2}{},
  cell{2}{6} = {c=2}{},
  cell{2}{8} = {c=2}{},
  hline{1,4} = {-}{},
  hline{2-3} = {2-9}{},
  colsep=3pt,
  stretch=0
}
\textbf{Method} & Mixed-Modal &  &  &  & Cross-Modal &  &  & \\
 & \textbf{Mix} &  & \textbf{Mix-ID} &  & \textbf{I$\rightarrow$V} &  & \textbf{\textbf{V$\rightarrow$I}} & \\
 & \textbf{R1} & \textbf{mAP} & \textbf{R1} & \textbf{mAP} & \textbf{R1} & \textbf{mAP} & \textbf{R1} & \textbf{mAP}\\
DDAG \cite{DDAG} &  99.9 & 76.17 & 45.29 & 45.54& 69.34 & 63.46 & 68.06  & 61.80\\
MPANet \cite{wu2021Nuances}  & \textbf{100} & 76.46 & 42.18 & 44.66 &  84.27& 80.20 & 83.20 & 79.82 \\
DEEN \cite{LLCM}&  99.95 & 83.06 & 66.21 & 59.61 &  90.29 &  83.98 & 91.21 & 85.13 \\ 
SAAI  \cite{SAAI}   & \textbf{100} & 82.29 & 58.01 & 55.41 & 86.21 & 80.0 & 86.60 &  81.51 \\ 
IDKL  \cite{IDKL}   & \textbf{100} & 83.83 & 61.46 & 60.54 & 87.23 & 83.20 & 87.91 &  85.07 \\
SAAI +AIM & - & - & - & - & 92.09 & 92.01 & 91.07 & 91.45\\
IDKL + KR & - & - & - & - & 94.22 & 90.43 & 94.72 & 90.19\\
\hline
\textbf{\NM} & \textbf{99.9} & \textbf{89.44} & \textbf{79.95} & \textbf{73.45} & 90.49 & 85.63 & 90.53 & 86.42\\
\textbf{\NM} + AIM & - & - & - & - & 90.78 & 90.18 & 90.35 & 90.02\\
\textbf{\NM} + KR & - & - & - & - & \textbf{97.09} & \textbf{93.01} & \textbf{96.55} & \textbf{93.55}\\

\hline
\end{tblr}
}
    \vspace{-0.3cm}
    \caption{RegDB dataset.}
    \label{tab:regdb}
  \end{subtable}%
  \hspace{0.2cm}
  \begin{subtable}{.48\textwidth}
  \centering
  
\resizebox{\textwidth}{!}{
\begin{tblr}{
  colspec={|l|cc|cc||cc|cc|},
  cell{1}{1} = {r=3}{},
  cell{1}{2} = {c=4}{c},
  cell{1}{6} = {c=4}{c},
  cell{2}{2} = {c=2}{},
  cell{2}{4} = {c=2}{},
  cell{2}{6} = {c=2}{},
  cell{2}{8} = {c=2}{},
  hline{1,4} = {-}{},
  hline{2-3} = {2-9}{},
  colsep=3pt,
  stretch=0
}
\textbf{Method} & Mixed-Modal &  &  &  & Cross-Modal &  &  & \\
 & \textbf{Mix} &  & \textbf{Mix-ID} &  & \textbf{I$\rightarrow$V} &  & \textbf{\textbf{V$\rightarrow$I}} & \\
 & \textbf{R1} & \textbf{mAP} & \textbf{R1} & \textbf{mAP} & \textbf{R1} & \textbf{mAP} & \textbf{R1} & \textbf{mAP}\\
DDAG \cite{DDAG} &  69.34 & 63.46 & 68.06  & 61.80 & 40.14 & 26.88 & 45.17 & 29.94\\
MPANet \cite{wu2021Nuances}  & 96.76& 53.73 & 38.79 & 29.03  & 46.48 & 30.96 & 52.15 & 37.00\\
DEEN \cite{LLCM}&  97.63 & \textbf{64.69} & 49.11 &  38.38  & 69.49 & 54.75&73.95 & 58.75\\ 

SAAI  \cite{SAAI}   & 96.61 & 60.05 & 40.86 &  31.27  & 59.37 & 45.65 & 64.37 & 48.60\\ 
IDKL  \cite{IDKL}   & 97.73 & 62.46 & 38.88 &  33.85  & 62.53 & 49.33 & 70.36 & 55.04 \\ 
IDKL + KR & - & - & - & - & 70.72 & 65.19 & 72.22 & 66.43\\ \hline
\textbf{\NM} & \textbf{97.80} & 64.45 & \textbf{57.1} & \textbf{45.71} & 65.76 & 51.08 & 70.79 & 56.61\\
\textbf{\NM }+ AIM & - & - & - & - & 66.11 & 62.89 & 74.10 & \textbf{68.20}\\
\textbf{\NM} + KR & - & - & - & - & \textbf{73.72} & \textbf{65.36} & \textbf{76.14} & 65.46\\
\hline
\end{tblr}
}
  \vspace{-0.3cm}
  \caption{LLCM dataset.}
  \label{tab:llcm}
  \end{subtable}%
  \vspace{-0.3cm}
  \caption{Performance of SOTA VI-ReID techniques in mixed, cross, and uni-modal settings on the (a) RegDB and (b) LLCM datasets.}
  \vspace{-.5cm}
\end{table*}


\noindent\textbf{(a) Mixed-Modal Results.}
Tables \ref{tab:SYSU}, \ref{tab:regdb}, and \ref{tab:llcm} show the performance of SOTA alongside our \NM method across various mixed-modal settings with an infrared query (visible query results are available in supplementary materials) on SYSU-MM01(single-shot), RegDB, and LLCM, respectively. These cross-dataset results highlight the effectiveness of our approach across different VI-ReID methods. In these mixed settings, the difficulty level increases from setting 1 to 4 due to the progressive reduction in positive samples within the same modality as the query, challenging each method's robustness in mixed scenarios. Notably, this analysis also enables us to examine each method's strengths and weaknesses in novel contexts; for instance, SGEIL \cite{shape-Erase23} performs best in the ``MIX'' setting but exhibits a significant performance drop in the ``Mix-ID'' setting on the SYSU-MM01 dataset compared to other methods.
The proposed \NM consistently outperforms existing methods across nearly all mixed settings, highlighting its capacity to bridge the modality gap between visible and infrared images by learning modality-erased features while enhancing same-modality matching through modality-related information. Importantly, modality-related features perform less effectively in Mix-ID and Cross-modal settings, where the same person does not appear in both modalities, limiting their usefulness for cross-modal matching. However, they help correct erroneous intra-modal matches by adjusting the similarity scores produced by modality-erased features. This supports our approach: modality-related information should be excluded from modality-erased features for robust inter-modal matching.

\begingroup
\tabcolsep = 3.0pt
\def\arraystretch{0.95}
\begin{table*}[!pb]
\centering
\begin{minipage}[t]{0.59\linewidth}
\centering
\resizebox{\linewidth}{!}{%
\begin{tabular}{|l|cc|cc|cc|cc|cc|}
\hline
\multirow{2}{*}{\textbf{Method}} & \multicolumn{2}{c|}{\textbf{Cross-modal}} & \multicolumn{2}{c|}{\textbf{Mix}} & \multicolumn{2}{c|}{\textbf{Mix-Cam}} & \multicolumn{2}{c|}{\textbf{Mix-Cam-ID}} & \multicolumn{2}{c|}{\textbf{Mix-ID}}  \\ \cline{2-11} 
 & \textbf{R1} & \textbf{mAP} & \textbf{R1} & \textbf{mAP} & \textbf{R1} & \textbf{mAP} & \textbf{R1} & \textbf{mAP} & \textbf{R1} & \textbf{mAP} \\ \hline
DDAG\cite{DDAG} & \Gr{53.29} & 50.61 & 91.91 & 57.60 & 80.81 & 53.54 & 70.41 & 45.14 & 23.92 & 29.09 \\ 
DDAG(f)+\NM & 53.21 & \Gr{50.86} & \Gr{92.04} & \Gr{60.29} & \Gr{81.27} & \Gr{55.07} & \Gr{72.05} & \Gr{48.68} & \Gr{35.71} & \Gr{34.81} \\ 
DDAG+\NM & \Bl{54.61} & \Bl{51.27} & \Bl{92.52} & \Bl{62.52} & \Bl{81.74} & \Bl{57.16} & \Bl{72.96} & \Bl{51.93} & \Bl{38.79} & \Bl{38.79} \\ \hline
MPANet\cite{wu2021Nuances} & 69.34 & 66.42 & 95.46 & 72.82 & 89.45 & 70.34 & 81.95 & 63.81 & 49.06 & 51.08 \\
MPANet(f)+\NM & \Gr{69.66} & \Gr{66.51} & \Gr{95.02} & \Gr{77.01} & \Gr{90.43} & \Gr{72.63} & \Gr{83.98} & \Gr{69.80} & \Gr{63.96} & \Gr{60.75} \\ 
MPANet+\NM & \Bl{70.93} & \Bl{68.18} & \Bl{95.80} & \Bl{78.67} & \Bl{91.34} & \Bl{74.98} & \Bl{83.95} & \Bl{70.55} & \Bl{66.00} & \Bl{61.16} \\ \hline
DEEN\cite{LLCM} & 72.55 & \Gr{68.59} & 94.07 & 74.06 & 87.50 & 72.09 & 79.50 & 65.66 & 56.77 & 55.30 \\ 
DEEN(f)+\NM & \Gr{72.63} & 68.49 & \Gr{94.55} & \Gr{78.91} & \Gr{88.38} & \Gr{75.27} & \Gr{82.77} & \Gr{70.9} & \Gr{65.8} & \Gr{60.05} \\ 
DEEN+\NM & \Bl{73.00} & \Bl{69.00} & \Bl{95.00} & \Bl{79.00} & \Bl{89.00} & \Bl{76.00} & \Bl{83.00} & \Bl{71.00} & \Bl{66.00} & \Bl{61.00} \\ \hline
SGEIL\cite{shape-Erase23} & 73.34 & 67.44 & 94.82 & 71.03 & 89.22 & 70.16 & 79.33 & 61.34 & 46.60 & 48.37 \\
SGEIL(f)+\NM & \Gr{73.40} & \Gr{67.81} & \Gr{94.90} & \Gr{76.45} & \Gr{90.14} & \Gr{73.43} & \Gr{82.63} & \Gr{68.86} & \Gr{64.14} & \Gr{60.15} \\ 
SGEIL+\NM & \Bl{74.08} & \Bl{69.19} & \Bl{95.33} & \Bl{78.01} & \Bl{91.16} & \Bl{74.29} & \Bl{83.97} & \Bl{69.40} & \Bl{65.37} & \Bl{61.50} \\ \hline
SAAI\cite{SAAI} & \Gr{73.87} & \Gr{69.71} & 96.01 & 74.59 & 90.63 & 72.51 & 84.30 & 65.94 & 52.49 & 53.30 \\
SAAI(f)+\NM & 73.64 & 69.51 & \Gr{96.21} & \Gr{79.12} & \Gr{91.7} & \Gr{75.03} & \Gr{87.0} & \Gr{71.22} & \Gr{63.72} & \Gr{60.91} \\ 
SAAI+\NM & \Bl{74.25} & \Bl{71.08} & \Bl{97.27} & \Bl{80.47} & \Bl{92.66} & \Bl{76.81} & \Bl{88.51} & \Bl{73.24} & \Bl{66.23} & \Bl{62.45} \\ \hline
IDKL\cite{IDKL} & 72.05 & 69.67 & 95.27 & 72.81 & 90.15 & 71.87 & 80.73 & 63.54 & 47.34 & 50.86 \\
IDKL(f)+\NM & \Gr{72.65} & \Gr{70.2} & \Gr{95.58} & \Gr{74.24} & \Gr{90.51} & \Gr{71.78} & \Gr{83.77} & \Gr{65.41} & \Gr{52.63} & \Gr{55.57} \\ 
IDKL+\NM & \Bl{73.44} & \Bl{71.02} & \Bl{96.10} & \Bl{76.67} & \Bl{91.38} & \Bl{74.27} & \Bl{85.61} & \Bl{68.44} & \Bl{58.37} & \Bl{57.91} \\ 
\hline
\end{tabular}
}

\vspace{-0.5cm}
\caption{Accuracy of the proposed method and SOTA methods as the baseline on the SYSU-MM01 (single-shot setting) as I query.
"(f)" indicates that our proposed losses were not back-propagated through the baseline models.}
\label{tab:mix-results2}
\end{minipage}
\begin{minipage}[t]{0.4\linewidth}
\resizebox{\linewidth}{!}{
\begin{tblr}{
  colspec={|l|ll|ll|ll|},
  cell{1}{2} = {c=2}{c},
  cell{1}{4} = {c=2}{c},
  cell{1}{6} = {c=2}{c},
  cell{2}{2} = {c=6}{c},
  cell{2}{1} = {r=2}{m},
  stretch = 0,
  colsep = 3.5pt
}
\hline
Source:       &    \textbf{SYSU-MM01}                    &                & \textbf{RegDB} &               & \textbf{LLCM}  &                \\ \hline
Target:     &  \textbf{Market1501(V$\rightarrow$V)}              &                &                &               &                &                \\ \hline
            &  \textbf{R1}                      & \textbf{mAP}   & \textbf{R1}    & \textbf{mAP}  & \textbf{R1}    & \textbf{mAP}   \\ \hline
 DDAG\cite{DDAG}       & 82.39 & 55.57 & 11.99 & 3.32 & 53.40  & 20.18  \\
MPANet\cite{wu2021Nuances}   & 78.97& 51.61 & \textbf{18.17} & 4.79 & 56.91& 22.79 \\
DEEN \cite{LLCM}& 66.50  & 33.81   & 6.85  & 1.7& 58.81 & 23.72 \\
SGEIL \cite{shape-Erase23}  & 79.18& 48.72& -   & -  & -& -    \\
SAAI \cite{SAAI}        & 84.32 & 57.62 & 14.54  & 3.44& 55.68 &22.77  \\
IDKL\cite{IDKL}     &   76.66    & 49.59  & & & &  \\ \hline
\textbf{\NM}(E)            & 82.90 & 56.56  & 14.31   & 3.97 & 54.45 &22.52          \\
\textbf{\NM}(R)      & 83.46  & 56.85   & 16.29  & 4.08 & 59.56 & 25.93          \\
\textbf{\NM}(E+R) & \textbf{87.93}   & \textbf{62.41} & 17.7  & \textbf{5.02} & \textbf{61.37} & \textbf{27.88} \\ \hline
 Upper-bound& 95.1                             & 87.8           & 95.1           & 87.8          & 95.1           & 87.8         \\  
\hline
\end{tblr}
}

\vspace{-0.5cm}
\caption{Performance of techniques of V-I ReID on Cross-Dataset RGB Market1501 dataset. Columns indicate the training dataset, and rows show the model’s performance on the Market1501.}
\label{tab:uni-result}

\centering
\vspace{0.2cm}
\resizebox{0.8\linewidth}{!}
{
\begin{tblr}{
  colspec={|ccccc||cc||cc|},
  cell{1}{1} = {c=5}{halign = c},
  cell{1}{6} = {c=2}{halign = c},
  cell{1}{8} = {c=2}{halign = c},
  stretch = 0,
  colsep = 3.5pt
}
\hline
\textbf{Settings}&  & & & &\textbf{Cross-Modal} & &\textbf{Mix-Cam-ID} & \\ \hline
$\mathcal{L}_\text{yme}$ & $\mathcal{L}_\text{ymr}$ &$\mathcal{L}_\text{o}$ &$\mathcal{L}_\text{m}$ &$\mathcal{L}_\text{f}$ & \textbf{R1} & \textbf{mAP} & \textbf{R1} & \textbf{mAP} \\ \hline
\cmark &   &   &   &   & 69.74  & 66.38  & 80.57  & 62.51 \\
\cmark & \cmark &   &   &   &  69.28 & 66.09  & 83.61  & 64.57 \\
\cmark&   & \cmark  &   &   & 70.25  & 67.11  & 81.07  & 62.99 \\
\cmark&   &  & \cmark  &   & 71.72  & 67.60  & 78.55  & 60.38 \\
\cmark& \cmark  & \cmark  & \cmark  &   & 73.40  & 70.87  & 84.64  & 65.83\\
\cmark & \cmark  & \cmark & \cmark & \cmark & \textbf{73.43} & \textbf{70.92}  & \textbf{87.56}  & \textbf{72.70} \\ \hline
\end{tblr}
}
\vspace{-0.3cm}
\caption{Impact of losses on \NM performance.}
\label{tab:losses}
\end{minipage}
\end{table*}
\endgroup

To further evaluate the adaptability of our learning paradigm, we applied it to several SOTA VI-ReID methods as a baseline on the SYSU dataset. We tested it in two variations: (a) stopping gradient backpropagation of our proposed modality-related and erased learning losses and (b) enabling end-to-end training with combined gradients from the existing and our proposed losses to strengthen cross-modal matching by filtering out modality-related features from the backbone. Results in Table \ref{tab:mix-results2} indicate that modality-related learning significantly boosts performance in mixed-modal settings. In contrast, end-to-end modality-erased learning enhances both mixed and cross-modal matching. For example, our modality-related learning improves the mAP of SAAI\cite{SAAI} and IDKL \cite{IDKL} in "Mix-ID" settings more than 6\% and 4\%, respectively. 

\noindent\textbf{(b) Cross-Modal Results.}
To show that erasing modality-specific information from features makes them more proper for cross-modal matching, we measure the performance of \NM and compare it with SOTA VI-ReID in Tables \ref{tab:SYSU}, \ref{tab:regdb}, and \ref{tab:llcm}. Our experiments show that \NM outperforms these methods in various situations. Modality-specific sub-modules are not triggered during inference in cross-modal matching. For example, compared to the second-best approach for the "All Search" scenario, \NM outperforms by a margin of 2.3\% R1 and 3.3\% mAP without adding complexity to the model.
For the "I\!$\rightarrow$\!V" mode on RegDB, \NM achieves 91.4\% R1 accuracy and 85.5\% mAP. For the "V$\rightarrow$I" mode, our method also obtains 90.4\% R1 accuracy and 86.8\% mAP. The results validate the effectiveness of our proposed \NM and show that it can effectively reduce the discrepancy between the V and I modalities.

In addition, our modality-erased feature learning has the advantage that it can be used in different VI-ReID models to improve their performance in cross-modal settings without incurring overhead during testing. To show this adaptability, in Table \ref{tab:mix-results2}, in the first column, end-to-end training improves the performance of all methods in the SYSU-MM01 dataset. For example, \NM increases the mAP of SGEIL\cite{shape-Erase23} and SAAI\cite{SAAI} by 1.75\% and 1.37\%.


\begin{figure*}[t!]
\centering
\begin{subfigure}{.172\textwidth}
  \centering
  \includegraphics[width=\linewidth, height=2cm]{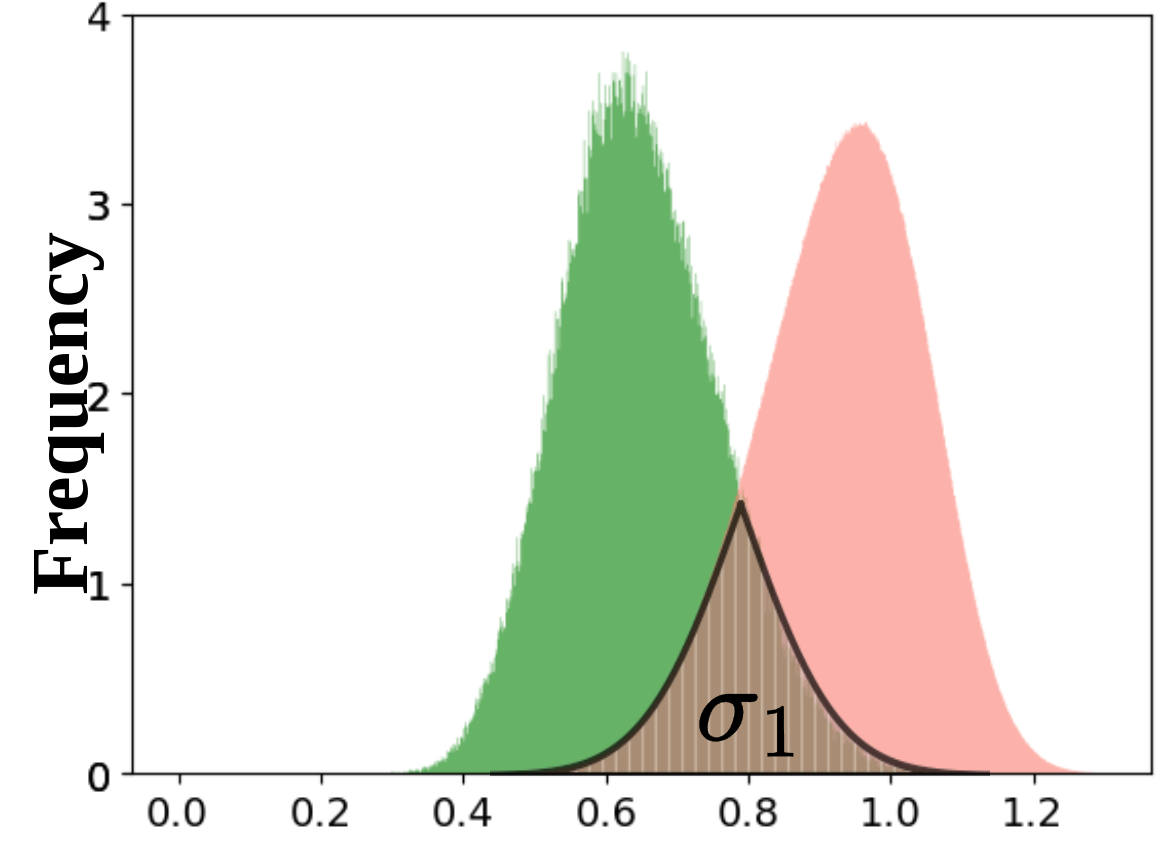}
  \vspace{-0.45cm}
  \caption{Baseline (Cross)}
  \label{fig:sub1}
\end{subfigure}%
\begin{subfigure}{.163\textwidth}
  \centering
  \includegraphics[width=\linewidth, height=2cm]{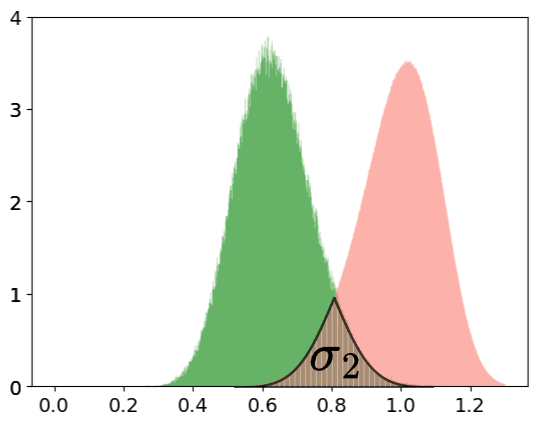}
  \vspace{-0.45cm}
  \caption{\NM (Cross)}
  \label{fig:sub1}
\end{subfigure}%
\begin{subfigure}{.163\textwidth}
  \centering
  \includegraphics[width=\linewidth, height=2cm]{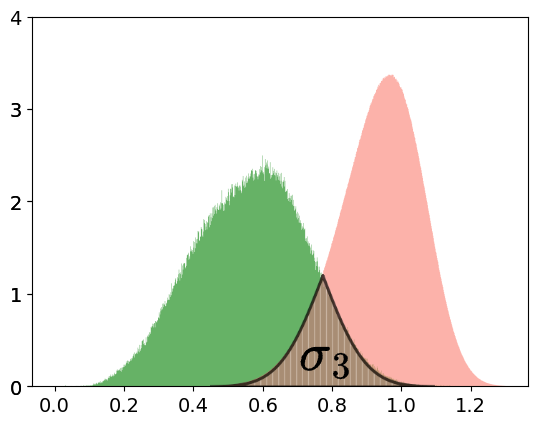}
  \vspace{-0.45cm}
  \caption{Baseline (Mix)}
  \label{fig:sub2}
\end{subfigure}%
\begin{subfigure}{.166\textwidth}
  \centering
  \includegraphics[width=\linewidth, height=2cm]{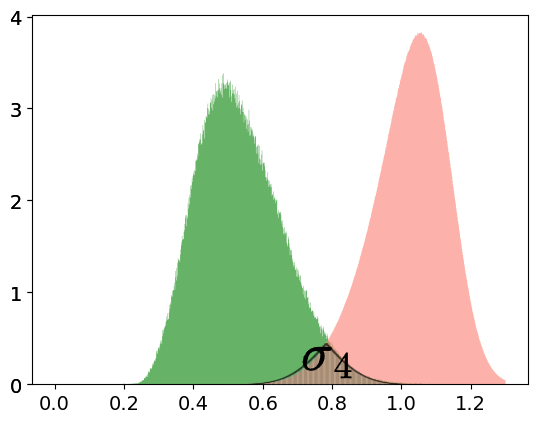}
  \vspace{-0.45cm}
  \caption{\NM (Mix)}
  \label{fig:sub2}
\end{subfigure}%
\begin{subfigure}{.166\textwidth}
  \centering
  \includegraphics[width=\linewidth, height=2cm]{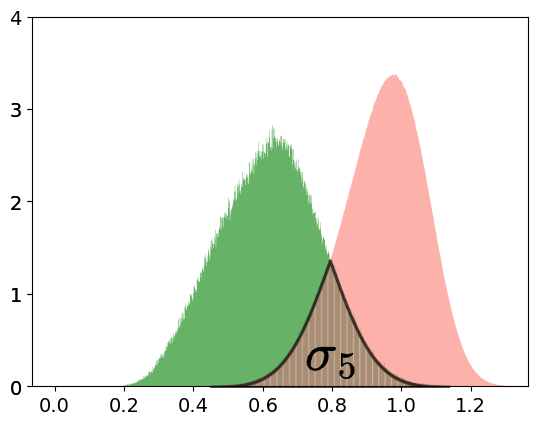}
  \vspace{-0.45cm}
  \caption{Baseline (Mix-Cam)}
  \label{fig:sub2}
\end{subfigure}%
\begin{subfigure}{.166\textwidth}
  \centering
  \includegraphics[width=\linewidth, height=2cm]{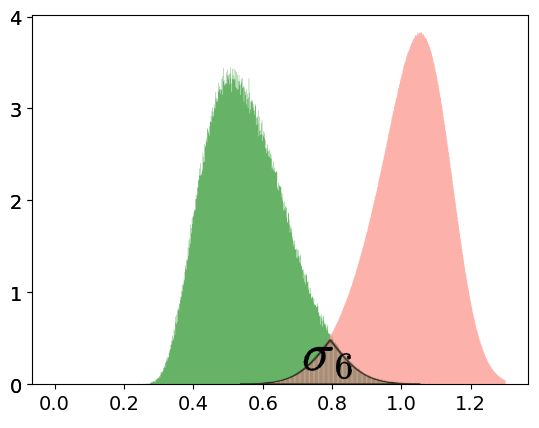}
  \vspace{-0.45cm}
  \caption{\NM (Mix-Cam)}
  \label{fig:sub2}
\end{subfigure}%
\\
\vspace{-0.3cm}
\caption{The intra-class (green) and inter-class (pink) distances distribution of features in different gallery settings. Here SAAI method \cite{SAAI} has been considered as the Baseline.
}
\label{fig:mixed-distance}
\vspace{-0.25cm}
\end{figure*}

\noindent\textbf{(c) Uni-Modal Results.}
To further illustrate the effectiveness of learning modality-related alongside modality-erased in \NM for uni-modal matching, we tested our trained model on the SYSU-MM01 dataset in V$\rightarrow$V and I$\rightarrow$I settings, as well as on an unseen visible dataset. 
Specifically, we used the Market1501 dataset \cite{market1501} to evaluate a model trained on the SYSU-MM01, RegDB, and LLCM datasets.
Table \ref{tab:SYSU} (column "Uni-Modal") presents these results for SYSU-MM01, demonstrating that our \NM approach outperforms other methods.
Table \ref{tab:uni-result} reports the R1 and mAP scores of various methods, while each column indicating the training dataset and each row showing the model’s performance on the Market1501 dataset.
Our \NM approach achieves the highest R1 and mAP scores, with values of 87.9\% and 62.4\%, respectively. When using only modality-related features (activating only the V branch) or only modality-erased features, the mAP scores are 56.8\% and 56.5\%, respectively. However, the integration of both feature types increases the mAP and R1 scores more than 3\%, demonstrating that this fusion improves generalization and provides a complementary representation of input images. 

\subsection{Ablation Studies}

\noindent{\textbf{Losses.}}
Table \ref{tab:losses} shows the impact of each loss component on performance. Using only the $\mathcal{L}_\text{y}$ as a baseline achieves results of 69.7\% R1 and 66.3\% mAP in cross-modal matching. Adding $\mathcal{L}_\text{ymr}$ improves Mix-Cam-ID R1 to 83.6\%, indicating $\mathcal{L}_\text{ymr}$’s benefit for intra-modality matching consistency, though cross-modal performance slightly decreases. Incorporating only $\mathcal{L}_\text{o}$ improves both cross- and mixed-modal performance, whereas $\mathcal{L}_\text{m}$ provides a more significant enhancement in Cross-Modal metrics. The combined effect of $\mathcal{L}_\text{y}$, $\mathcal{L}_\text{yme}$, $\mathcal{L}_\text{o}$, and $\mathcal{L}_\text{m}$ achieves 73.4\% R1 in cross-modal, showing these losses’ synergy. Finally, adding $\mathcal{L}_\text{f}$ achieves optimal results across all metrics, highlighting its important role in refining features for mixed-modal settings.

\begin{figure}[!th]
    \centering
    \includegraphics[width=\linewidth]{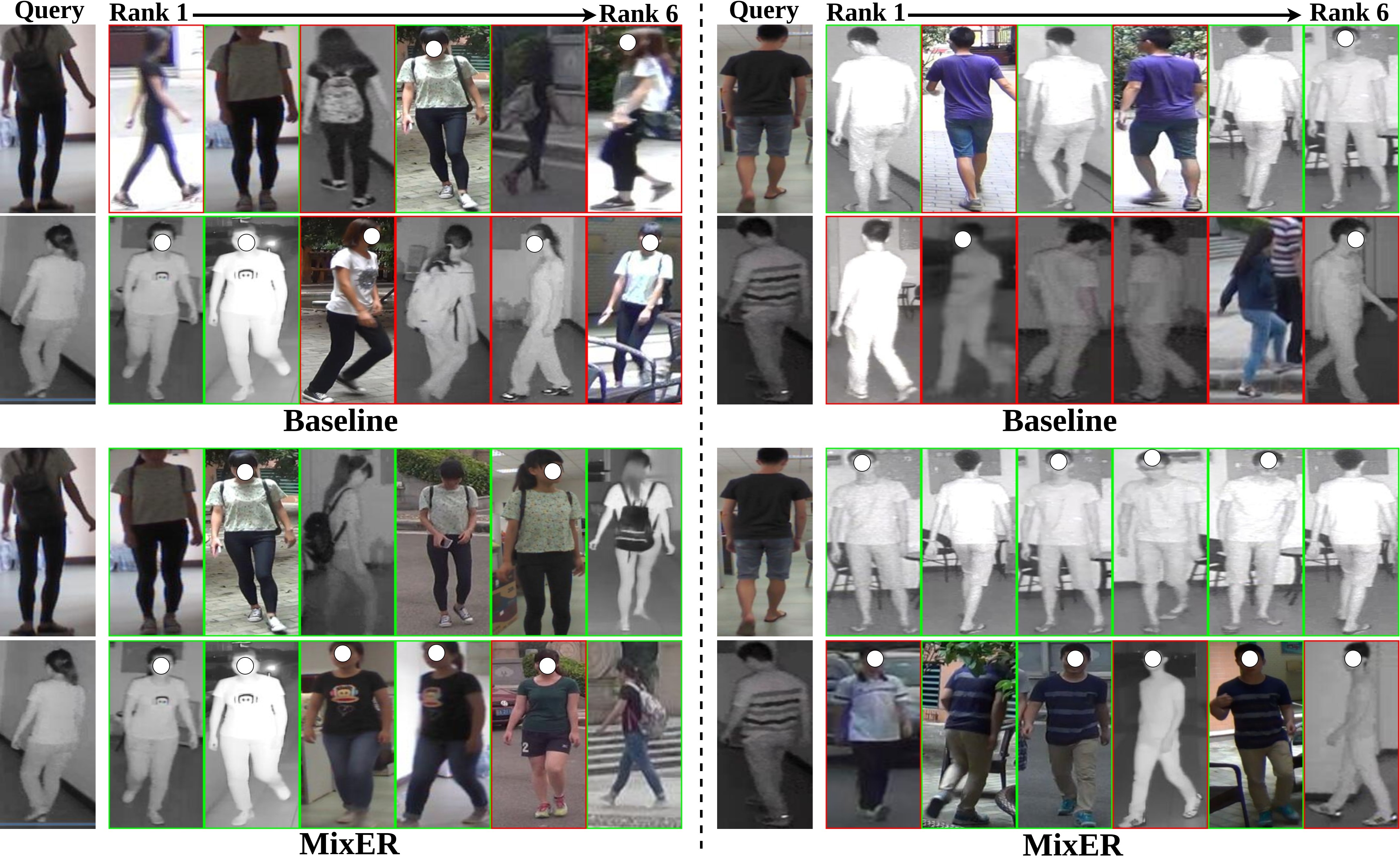}
    \vspace{-.75cm}
    \caption{Rank-6 retrieval results obtained by the baseline (top) and the proposed \NM (bottom) on SYSU dataset under Mix-Cam-ID (left) and Mix-ID (right) settings.}
    \label{fig:rank}
    \vspace{-0.65cm}
\end{figure}

\subsection{Qualitative Analysis}
\vspace{-0.2cm}
{\noindent \textbf{Feature distribution.}}
To examine the effectiveness of our method in cross-modal and mixed-modal matching, we visualize inter-class and intra-class distance distributions on the SYSU dataset, as shown in Fig.\ref{fig:mixed-distance}(a-f). Compared to the baseline, \NM reduces both the mean and variance of intra-class distances across all gallery settings, decreasing the area of overlap ($\sigma_2\!<\!\sigma_1$, $\sigma_4\!<\!\sigma_3$, and $\sigma_6\!<\!\sigma_5$). This demonstrates that \NM more effectively reduces intra-class distances, thereby decreasing modality discrepancy in modality-erased features and enhancing identity discrimination in modality-related features. Furthermore, UMAP visualizations (see supplementary material) show that \NM better separates identities and reduces modality discrepancy within the mixed-modal gallery.

{\noindent \textbf{Retrieval result.}}
To further show the effectiveness of \NM, we present retrieval results from our model on the SYSU-MM01 dataset in Fig.\ref{fig:rank}. 
In the retrieval results, green boxes indicate correct matches, while red boxes represent incorrect ones.
Overall, combining modality-erased and modality-related features at the bottom significantly improves the ranking results, with more correct matches appearing in the higher rank positions compared to the baseline at the top. While modality-erased learning aims to find the closest matches based solely on shared attributes between visible and infrared images, it may mistakenly select incorrect images. Modality-related features refine this by penalizing incorrect matches based on attributes specific to the query modality.
For example, in the second row at the top, the model ranks images of a person wearing a T-shirt and shorts, albeit in different colors, as the top matches. However, in the second row at the bottom, the model produces better matches by leveraging modality-erased and modality-related information.

\vspace{-0.2cm}
\section{Conclusion}
\vspace{-0.2cm}
In this paper, we proposed a new evaluation setting, Mix-Modal ReID, to address the real-world challenge of person re-identification in mixed visible-infrared galleries, which more accurately reflects complex surveillance conditions. We introduced a modality-erased and modality-related feature learning approach to improve robustness across modalities. Our experiments on SYSU-MM01, RegDB, and LLCM demonstrated that our method outperforms current state-of-the-art VI-ReID methods in both mixed- and cross-modal settings, highlighting the limitations of existing methods under these new conditions. Our approach thus provides a promising foundation for more effective and adaptable VI-ReID systems in real-world applications.

{
    \small
    \bibliographystyle{ieeenat_fullname}
    \bibliography{main}
}
\clearpage
\newtheorem{theorem}{Theorem}

\newtheorem{hyp}{Hypothesis}
\newtheorem{Property}{P}

\appendix
\numberwithin{equation}{section}
\numberwithin{figure}{section}
\numberwithin{table}{section}
\setcounter{page}{1}
\setcounter{thm}{0}

\maketitlesupplementary

\begin{equation}
\begin{aligned}
    \max_{Z^e_{m},Z^r_{m}} \{ \overbrace{ \I(Z^e_{m};Y|M)\!- \!\lambda_1\I(Z^e_{m}; M)}^{\textit{Modality-Erased Learning} } + 
    \!\overbrace{\I(Z^r_{m};Y;M)}^{\textit{Modality-Related Learning}}  - \overbrace{\lambda_2 \I(Z^e_{m}; Z^r_{m})}^{\text{Orthogonal Feature Learning}}\}.
\end{aligned}
\end{equation}

\section{Proofs:}
\label{sec:proofs}

In Section \ref{sec:proposed}, our model is designed to represent input images through two distinct and complementary feature types: modality-related and modality-erased representations. To ensure the independence of these representations, we minimize the mutual information (MI) between the data distributions of them. This is achieved by applying an orthogonal loss to their feature representations. The modality-related features are intended to capture ID-aware information specific to the modality, whereas the modality-erased features are explicitly designed to exclude any modality-specific information.
To ensure that the extracted features satisfy these conditions, we use constrained optimization to maximize the mutual information between the joint distribution of modality-related and modality-erased features and the label distribution. In this section, we prove that our proposed loss functions meet these constraints and align with the properties of mutual information.

\subsection{‌Backgrounds}
The following highlights the main properties of mutual information:

\begin{Property} [{\bf Nonnegativity}]
\label{p1}
For every pair of random variables $X$ and $Y$:
\begin{equation}
    \I(X;Y) \geq 0
\end{equation}
\end{Property}
\begin{Property} \label{p2}
For random variables $X$, $Y$ that are independent:
\begin{equation}
    \I(X;Y) = 0.
\end{equation}
\end{Property}

\begin{Property} [{\bf Monotonicity}] \label{p3}
For every three random variables $X$, $Y$ and $Z$:
\begin{equation}
    \I(X;Y;Z) \leq \I(X;Y)
\end{equation}
\end{Property}

\begin{Property} \label{p4}
For every three random variables $X$, $Y$ and $Z$, the mutual information of joint distribution $X$ and $Z$ to $Y$ is (Fig. \ref{fig:s_mi1}):
\begin{equation}
    \I(X,Z;Y) = \I(X;Y) + \I(Z;Y) - \I(X;Z;Y)
\end{equation}
\end{Property}

\begin{figure}[h!]
    \centering
    \begin{subfigure}{.48\linewidth}
      \centering
      \includegraphics[height=0.5\linewidth]{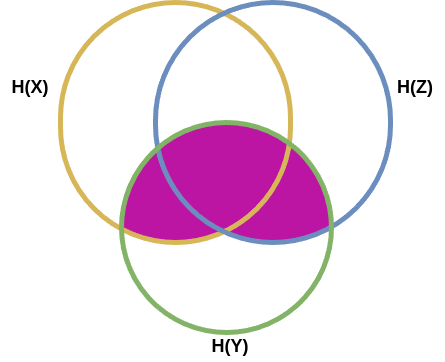}
      \caption{$\I(X,Z;Y)$}
      \label{fig:s_mi1}
    \end{subfigure} %
    \hspace{0.2cm}
    \begin{subfigure}{.48\linewidth}
      \centering
      \includegraphics[height=0.5\linewidth]{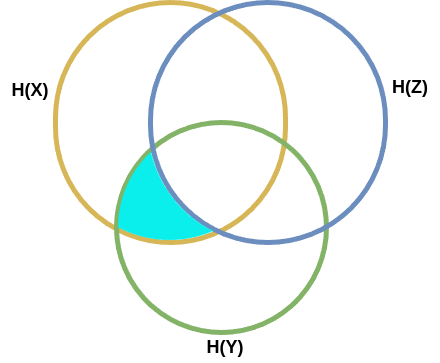}
      \caption{$\I(X;Y|Z)$}
      \label{fig:s_mi2}
    \end{subfigure}
    \vspace{-0.15cm}
    \caption{Venn diagram of theoretic measures for three variables $X$, $Y$, and $Z$, represented by the left, right and bottom circles, respectively.    }
    \label{fig:s_mi}
\end{figure}


\begin{Property} [{\bf Conditional}] \label{p5}
For every three random variables $X$, $Y$ and $Z$, the conditional mutual information is(Fig. \ref{fig:s_mi2}):
\begin{equation}
    \I(X;Y|Z) = \I(X;Y) - \I(X;Y;Z)
\end{equation}
\end{Property}

Also, we described some hypotheses that are used in our learning:
\begin{hyp} \label{hyp:first}
Following \cite{shape-Erase23}, the orthogonality between $\mathbf{z}_m^r$ and $\mathbf{z}_m^e$ can be regarded as a relaxation of independence:
\begin{equation*}
    \forall (\mathbf{z}_m^r, \mathbf{z}_m^e) \sim (Z_m^r, Z_m^e), \ \mathbf{z}_m^r, \mathbf{z}_m^e \perp \mathbf{z}_m^e \ \Rightarrow \ \I (Z_m^r;Z_m^e) \simeq 0
\end{equation*}
\end{hyp}

\begin{hyp} \label{hyp:second}
Following \cite{shape-Erase23, MI2}, we posit that if \( Z \) is a representation of \( X \), then \( Z \) is conditionally independent of all other variables in the system given \( X \). This is formally expressed as:
\begin{equation}
    \forall A, B \quad I(A; Z \mid X, B) = 0.
\end{equation}
\end{hyp}

\begin{defn} [{\bf Sufficiency}] \label{def:suff}
A representation Z of X is sufficient for Y if and only if:
\begin{equation}
    \I(X;Y|Z) = 0 \Longleftrightarrow \I(X;Y) = \I(Z;Y)
\end{equation}
\end{defn}

Any model with access to a sufficiently informative representation Z must be able to predict Y with at least the same level of accuracy as if it had access to the original data X. Representation Z is considered sufficient for Y if and only if the task-relevant information remains unchanged during the encoding process.
If the cross-entropy loss between Z and Y is minimized, then, as suggested in \cite{MI2}, Z can be assumed to be a sufficient representation of X for the task Y.

\begin{defn} [{\bf Data Processing Inequality}] 
Let three random variables form the Markov chain $Y \rightarrow X\rightarrow Z$, implying that the conditional distribution of $Z$ depends only on $X$ and is conditionally independent of $X$, we have:
\begin{equation}
    \I(X;Y) \geq \I(Z;Y).
\end{equation}
\end{defn}
The data processing inequality (DPI) is a fundamental inequality in information theory that states the mutual information between two random variables cannot increase through processing.

\begin{thm} \label{thm:1m_supp} If $Z^e_{m}$ and $Z^r_{m}$ are independent, then $\I(Z^e_{m},Z^r_{m};Y)=\I(Z^e_{m};Y) + \I(Z^r_{m};Y)$.
\end{thm}
\begin{proof}
For independent random variables $Z^e_{m}$ and $Z^r_{m}$, we have the following relationship:
\begin{equation*}
    \I(Z^e_{m}; Z^r_{m}) = 0.
\end{equation*}
Also, w.r.t to P \ref{p4}:
\begin{equation*}
    \I(Z^r_m,Z^e_m;Y) = \I(Z^r_m;Y) + \I(Z^e_m;Y) - \I(Z^r_m;Z^e_m;Y),
\end{equation*}
and P \ref{p3}:
\begin{equation*}
    \I(Z^r_m;Z^e_m;Y) \leq \I(Z^r_m;Z^e_m) = 0,
\end{equation*}
so, we have:
\begin{equation}
    \I(Z^e_{m},Z^r_{m};Y)=\I(Z^e_{m};Y) + \I(Z^r_{m};Y)
\end{equation}
\end{proof}

\subsection{Minimizing $\I(Z^e_{m};M)$}
In order to minimizing $\I(Z^e_{m};M)$, we use adversarial learning approach and convert $Z^e_{m} \in \mathcal{Z}^e_m$ to $Z^o_{m} \in \mathcal{Z}^o_m$ with deterministic and learnable function $\mathcal{W} : \mathcal{Z}^e_m \rightarrow \mathcal{Z}^o_m$:
\begin{equation}
    Z^o_{m} = \mathcal{W}(Z^e_{m}; \theta_W).
\end{equation}
Then, we maximize the $\I(Z^o_{m};M)$ with a standard SGD approach through the optimization of $\theta_W$ and reversed SGD through the model's parameters. Maximizing $\I(Z^o_{m};M)$ is achieved by minimizing cross-entropy loss between the estimated modality label from $Z^o_{m}$ as $\hat{m}$ and ground truth label, $m$ ($\mathcal{L}_\text{m}$ in main manuscript). This approximation is formulated in {\bf Proposition \ref{prop:cross}}. 
\begin{prop}
\label{prop:cross}
Let $Z$ and $Y$ be random variables with domains $\mathcal{Z}$ and $\mathcal{Y}$, respectively. Minimizing the conditional cross-entropy loss of predicted label $\hat{Y}$, denoted by $\mathcal{H}(Y; \hat{Y}|Z)$, is equivalent to maximizing the $\I(Z; Y)$
\end{prop}
\begin{proof}
    Let us define the MI as entropy,
    \begin{equation}
        \I(Z,Y) = \underbrace{\mathcal{H}(Y)}_{\delta} - \underbrace{\mathcal{H}(Y|Z)}_{\xi}
    \end{equation}
    Since the domain $\mathcal{Y}$ does not change, the entropy of the identity $\delta$ term is a constant and can therefore be ignored. Maximizing $\I(Z,Y)$ can only be achieved through a minimization of the $\xi$ term. We show that $\mathcal{H}(Y|Z)$ is upper-bounded by the cross-entropy loss, and minimizing such loss results in minimizing the $\xi$ term. By expanding its relation to the cross-entropy \cite{boudiaf2020unifying}:
    \begin{equation}
        \label{eq:cross}
        \mathcal{H}(Y; \hat{Y}|Z) = \mathcal{H}(Y|Z) + \underbrace{\mathcal{D}_{\text{KL}}(Y||\hat{Y}|Z)}_{\geq 0} , 
    \end{equation}
    where we have:
    \begin{equation}
        \mathcal{H}(Y|Z) \leq \mathcal{H}(Y; \hat{Y}|Z),
    \end{equation}
    where minimizing $\mathcal{H}(Y; \hat{Y}|Z)$ results minimizing $\mathcal{H}(Y|Z)$.
\end{proof}

\subsection{Maximizing $\I(Z^e_{m};Y|M)$}
Eq. \ref{eq:ym} of main manuscript can be rewritten w.r.t P \ref{p5} as :
\begin{equation}
\label{eq:mi_ye}
    \I(Z^e_{m};Y|M) = \I(Z^e_{m};Y) - \I(Z^e_{m};Y;M).
\end{equation}
For the second part RHS:
\begin{equation*}
    \I(Z^e_{m};Y;M) \leq \I(Z^e_{m};M),
\end{equation*}
where the $\mathcal{L}_\text{m}$ in main manuscript is minimized, results $\I(Z^e_{m};M) \simeq 0$ and :
\begin{equation*}
    \I(Z^e_{m};Y;M) \simeq 0.
\end{equation*}
So, the Eq. \ref{eq:mi_ye} is:
\begin{equation}
    \I(Z^e_{m};Y|M) \simeq \I(Z^e_{m};Y).
\end{equation}
To maximize the $\I(Z^e_{m};Y)$, the $\mathcal{L}_\text{meid}$ is minimized w.r.t to {\bf Proposition \ref{prop:cross}}.

\subsection{Minimizing $\I(Z^r_{m};Y|M)$}
To enforce the modality-related features \( Z^r_{m} \) to leverage identity-aware information that is dependent on modality (i.e., specific identity-discriminative information), we minimize the amount of identity-aware information in these features that disregards the modality. Below, we demonstrate that \( \I(Z^r_{m}; Y \mid M) \) is upper-bounded by zero if the features \( Z^r_{m} \) are sufficient for both tasks \( Y \) (identity) and \( M \) (modality) simultaneously. To ensure that the modality-related features \( Z^r_{m} \) serve as sufficient representations of the input images \( X \) for both detecting modality and identifying identity, the loss function \( \mathcal{L}_{\text{mrid}} \) (as defined in the main manuscript) is applied to these features.

\begin{thm} \label{thm:2}
    If the representation \( Z^r_{m} \) of \( X \) is sufficient for both \( Y \) and \( M \), then:
    \begin{equation}
        \I(Z^r_{m}; Y \mid M) = 0.
    \end{equation}
\end{thm}

\begin{proof}
    From the definition of a sufficient representation \( Z^r_{m} \) for a task \( M \) (Definition \ref{def:suff}), we have:
    \begin{equation}\label{eq:Z_Y}
        \I(X; M \mid Z^r_{m}) = 0 \Longleftrightarrow \I(X; M) = \I(Z^r_{m}; M) \Rightarrow \I(Z^r_{m}; M \mid X) = 0.
    \end{equation}
    Similarly, for task \( Y \), we have:
    \begin{equation} \label{eq:Z_M}
        \I(Z^r_{m}; Y \mid X) = 0.
    \end{equation}
    Expanding \( \I(Z^r_{m}; Y \mid X) \), we obtain:
    \begin{equation} \label{eq:fff}
        \I(Z^r_{m}; Y \mid X) = \I(Z^r_{m}; Y; M \mid X) + \I(Z^r_{m}; Y \mid X, M) = 0.
    \end{equation}
    For the first term on the RHS of Eq.~\ref{eq:fff}, we have:
    \begin{equation*}
        \I(Z^r_{m}; Y; M \mid X) = \I(Z^r_{m}; M \mid X) - \I(Z^r_{m}; M \mid X, Y) = 0,
    \end{equation*}
    where \( \I(Z^r_{m}; M \mid X) = 0 \) follows from Eq.~\ref{eq:Z_M}, since \( Z^r_{m} \) is a sufficient representation of \( X \) for task \( M \), and \( \I(Z^r_{m}; M \mid X, Y) = 0 \) is due to Hypothesis \ref{hyp:second}.
    For the second term on the RHS of Eq.~\ref{eq:fff}, we have:
    \begin{equation}
        \I(Z^r_{m}; Y \mid X, M) = \I(Z^r_{m}; Y \mid M) - \I(Z^r_{m}; Y; X \mid M) = 0.
    \end{equation}
    Since \( \I(Z^r_{m}; Y; X \mid M) \leq \I(Z^r_{m}; Y \mid M) \), and \( \I(Z^r_{m}; Y \mid M) = \I(Z^r_{m}; Y; X \mid M) \), it follows that:
    \begin{equation}
        \I(Z^r_{m}; Y; X \mid M) = 0.
    \end{equation}
    Thus, \( \I(Z^r_{m}; Y \mid M) = 0 \), completing the proof.
\end{proof}

\section{Additional Experiments}
\subsection{Implementation Details of Open-Source SOTA Methods}

In this section, we present implementation details for each open-source state-of-the-art (SOTA) method used in our manuscript. For each method, we use the hyperparameter based on their paper or official code in GitHub:

\begin{itemize}
    \item {\bf DDAG} \cite{DDAG}: This method employs ResNet-50 as the backbone with stride one at the last layer, with input images resized to \(288 \times 144\) pixels that are augmented with zero-padding and horizontal flipping. Based on the DDAG paper\cite{DDAG}, 8 people with 4 V and 4 I images are selected in the batch, and $p=3$ is set. 

    \item {\bf DEEN} \cite{LLCM}: A modified ResNet-50 is used as the backbone, enhanced with DEE modules that introduce two additional branches to the network. Input images are resized to \(344 \times 144\) pixels. During training, augmentations such as Random Erase and Random Channel augmentation are applied. At inference time, the original image and its horizontally flipped counterpart are both processed through the backbone, and the average of the extracted features is used as the final representation. For evaluation under our mixed-modal settings, we removed the flipping process and instead concatenated the features from all DEE branches to create the final representative features.

    \item {\bf MPANet} \cite{wu2021Nuances}: This method also employs ResNet-50 as the backbone, with an additional convolutional layer designed to detect more discriminative regions in the feature space. The final representative feature vector is constructed by concatenating part features and global features obtained from a Global Average Pooling (GAP) layer. Input images are resized to \(344 \times 144\) pixels, and augmentations such as Random Erase and Random Channel augmentation are applied during training.

    \item {\bf SGEIL} \cite{shape-Erase23}: This method employs two ResNet-50 backbones, one for visible images and the other for infrared images, along with an additional ResNet-50 backbone for shape images. Input images are resized to \(288 \times 144\) pixels, with augmentations similar to those used in DEEN. During training, model weights are updated using SGD, while an exponential moving average (EMA) is simultaneously applied to update the backbone weights. At the end of training, the best performance between the SGD and EMA models is selected. Note that SGEIL requires shape images for training, and since this information is available only for the SYSU-MM01 dataset, its performance on RegDB and LLCM is not reported.

    \item {\bf SAAI} \cite{SAAI}: Similar to MPANet, this method uses ResNet-50 as the backbone, with an additional convolutional layer and learnable parameters to identify part prototypes. The final representative feature vector is obtained by concatenating part features and global features from the GAP layer. Input images are resized to \(288 \times 144\) pixels, and Random Erase and Random Channel augmentation are applied during training. An Affinity Inference Module is used during inference to rerank the gallery.

    \item {\bf IDKL} \cite{IDKL}: This method uses ResNet-50 as the backbone, with additional branches added to layers 3 and 4 to extract modality-specific features. Input images are resized to \(388 \times 144\) pixels, and Random Erase and Random Channel augmentation are applied during training. During inference, k-reciprocal encoding is applied to rerank the gallery, enhancing the retrieval process.
\end{itemize}

\subsection{Additional Mixed-Modal Results}
In the main manuscript, we reported the performance of mixed-modal settings for existing datasets with infrared query images and mixed-modal gallery images. Table \ref{tab:mix-results-v}, presents the performance with visible query. Across all mixed settings in SYSU-MM01, \NM consistently outperforms the compared methods, achieving the highest Rank-1 accuracy (R1) and mean Average Precision (mAP) scores. Specifically, for the Mix setting, our method achieves a notable improvement in mAP (87.29\%) compared to the next best method, IDKL (84.78\%), showcasing its robustness in handling mixed gallery conditions. In more challenging settings like Mix-Cam and Mix-ID, \NM{} significantly outperforms the SOTA, demonstrating its ability to adapt to modality and identity constraints effectively.

In the RegDB dataset, which focuses on visible-infrared retrieval, \NM{} achieves the highest scores across both the Mix and Mix-ID settings. The proposed method achieves a remarkable mAP of 92.78\% in the Mix setting, outperforming the best-performing baseline (IDKL) by over 4.4\%. Similarly, for the Mix-ID setting, \NM{} achieves a substantial improvement in mAP (81.42\%) compared to SAAI (68.82\%). On the LLCM dataset, \NM{} maintains competitive performance, achieving the highest scores in both the Mix and Mix-ID settings. In particular, \NM{} improves mAP in the Mix-ID setting to 45.18\%, which surpasses the previous best method, DEEN, by a notable margin (43.65\%). These results highlight the generalizability and robustness of our method in addressing varying cross-modal and identity constraints.

The proposed \NM{} consistently achieves superior performance across all datasets and settings, highlighting its effectiveness in extracting discriminative modality-related and modality-erased features. The substantial improvements in challenging settings, such as Mix-ID, underscore the effectiveness of the disentangling strategy employed by \NM, which allows it to handle both modality and identity variations effectively. Unlike competing methods, \NM achieves a balanced improvement in both Rank-1 accuracy and mAP, demonstrating its robustness in retrieval precision and ranking performance.


\begin{table*}[!h]
\centering
\resizebox{\textwidth}{!}{%
\begin{tabular}{|c|cc||cc||cc||cc||cc||cc||cc||cc|}
\hline
\multicolumn{1}{|l|}{} & \multicolumn{8}{c||}{SYSU-MM01} & \multicolumn{4}{|c||}{RegDB} & \multicolumn{4}{|c|}{LLCM} \\ \hline
\multirow{2}{*}{\textbf{Method}} & \multicolumn{2}{c||}{\textbf{Mix}} & \multicolumn{2}{c||}{\textbf{Mix-Cam}} & \multicolumn{2}{c||}{\textbf{Mix-Cam-ID}} & \multicolumn{2}{c||}{\textbf{Mix-ID}} & \multicolumn{2}{c||}{\textbf{Mix}} & \multicolumn{2}{c||}{\textbf{Mix-ID}} & \multicolumn{2}{c||}{\textbf{Mix}} & \multicolumn{2}{c|}{\textbf{Mix-ID}} \\ \cline{2-17} 
 & \textbf{R1} & \textbf{mAP} & \textbf{R1} & \textbf{mAP} & \textbf{R1} & \textbf{mAP} & \textbf{R1} & \textbf{mAP} & \textbf{R1} & \textbf{mAP} & \textbf{R1} & \textbf{mAP} & \textbf{R1} & \textbf{mAP} & \textbf{R1} & \textbf{mAP}\\ \hline
DDAG \cite{DDAG}  &  97.59 &79.62& 94.68& 76.61& 92.96& 73.42& 41.09& 45.85&  99.9 & 77.53 & 45.39 & 47.13 & 99.24 & 51.31 & 22.50 & 17.94  \\ 
MPANet \cite{wu2021Nuances} & \textbf{97.94} & 83.80 & 95.23 & 80.98 & 94.10 & 78.91 & 54.54 & 57.24                & \textbf{100} & 84.32 & 60.24 & 61.57 &  99.17& 59.61 & 25.08 & 18.82 \\
DEEN \cite{LLCM}    &                        95.79 & 82.18& 92.29& 80.32& 89.85 &77.32& 59.13& 61.03& 99.95 & 88.49 & 75.24 & 71.45 & 99.25 & 73.73 & 57.90 &  43.65 \\ 
SGEIL \cite{shape-Erase23} &96.76 & 80.52& 94.05 &78.74& 91.10& 74.82& 48.72& 52.96& - & - & - & - & - & - & - &  - \\ 
SAAI  \cite{SAAI}                         &97.63& 83.88& 95.22& 81.50& 93.82& 79.08& 55.08& 57.61& 100 & 88.19 & 69.22& 68.82& 99.57 &70.60& 46.12 &34.63 \\ 
IDKL  \cite{IDKL}                         & 98.25 & 84.78&  96.15&  82.51&  94.87&  79.85&  54.00&  57.51&  99.95 & 88.33 & 71.70 & 70.64 & 99.57 & 70.54 & 39.25 &  32.04 \\ 

 \hline
\NM(ours) & 97.14 & \textbf{87.29}& \textbf{96.27}& \textbf{85.67}& \textbf{94.95}& \textbf{84.79}& \textbf{70.96}& \textbf{70.76}& \textbf{100} & \textbf{92.78} & \textbf{85.39} & \textbf{81.42} & \textbf{99.80} & \textbf{74.59} & \textbf{58.87} & \textbf{45.18} \\ \hline 
\end{tabular}
}
\vspace{-0.3cm}
\caption{Accuracy of the proposed method and open-sourced state-of-the-art methods on the SYSU-MM01 (single-shot setting), RegDB, and LLCM datasets in different mixed gallery settings. Visible images are chosen as the query. }
\label{tab:mix-results-v}
\end{table*}

\subsection{Additional Ablation Studies}
\subsubsection{Computational Complexity}

We compare the training times of several state-of-the-art methods in Cross-Modal ReID to demonstrate how much additional computational burden is added due to the augmentation of these methodologies with our method. The training times have been documented in \autoref{tab:time}. Note that each method has its own variable number of training epochs, keeping in mind the optimal number of epochs suggested for these methods. We argue that the overall increase in computational time is a reasonable trade-off with a projected performance increase, thus highlighting the superiority of our method. 
The upper-bound model represents a best-case scenario with three separate SAAI models trained independently for visible, infrared, and VI images.
Also, to compare with the upper-bound, using \NM is more efficient.

\begin{table}[!htbp]
\centering
\resizebox{0.9\columnwidth}{!}{%
\begin{tabular}{|c|c|c|c|c|c|c|c|}
\hline
\textbf{Method} & \textbf{Training time} & \textbf{\begin{tabular}[c]{@{}c@{}}Training time \\ with \NM \end{tabular}} & \textbf{\# epochs}  & \textbf{\# of param}  & \textbf{\begin{tabular}[c]{@{}c@{}}\# of params \\ with \NM \end{tabular}} & \textbf{Flops} & \textbf{\begin{tabular}[c]{@{}c@{}}Flops \\ with \NM \end{tabular}}  \\ \hline
DDAG   &  111  &  137  &  80 & 40M & 54M & 0.5T & 0.6T  \\
DEEN   & 686  &  732  & 151  & 61M  & 75M & 1.3T &  1.5T   \\
SGEIL &  597  & 631    &  120  & 87M & 101M & 0.9T &  0.97T  \\
SAAI   & 78 & 101 & 160  & 72M & 85M& 0.75T & 0.8T \\
IDKL   & 182 & 204 & 180  & 88M & 106M & 0.95T & 1.1T \\
MPANet & 144  &  179   &  140  & 74.8M & 88M & 0.9T &  1.02T   \\ \hline
Upper-Bound & 184 & - & 160 & 216M & - & 2.25T & - \\ \hline
\end{tabular}%
}
\caption{Training times for state-of-the-art Cross-Modal ReID methods (on SYSY-MM01 dataset) compared to the training times of the same methods modified with our method. Time has been reported in minutes.}
\label{tab:time}
\end{table}

Table \ref{tab:SYSU_Upper} compares the performance of the baseline SAAI method \cite{SAAI}, its modified version enhanced with our proposed \NM{} framework, and an upper-bound model across mixed-modal, cross-modal, and uni-modal settings on the SYSU-MM01 dataset. The results show that the SAAI+\NM{} model consistently outperforms the baseline SAAI in mixed-modal settings, achieving notable improvements in both Rank-1 accuracy and mAP. For instance, in the Mix setting, SAAI+\NM{} achieves an mAP of 80.47\% compared to 74.59\% for the baseline. In the challenging Mix-ID setting, it improves mAP from 53.30\% to 62.45\%, demonstrating the effectiveness of \NM{} in disentangling modality-related and modality-erased features.

In cross-modal settings, SAAI+\NM{} also demonstrates robustness, achieving a slight improvement in mAP for the "All" setting (71.08\% vs. 69.71\%) while maintaining competitive performance in uni-modal scenarios. In the I$\rightarrow$I and V$\rightarrow$V settings, SAAI+\NM{} either matches or outperforms the baseline without compromising intra-modality retrieval. Notably, in some cases, SAAI+\NM{} exceeds the upper-bound, such as in the Mix setting where it achieves an mAP of 80.47\

\begin{table*}[!pht]
\centering
\resizebox{0.95\textwidth}{!}{%
\begin{tblr}{
  colspec={|l|cc|cc|cc|cc||cc|cc||cc|cc|},
  cell{1}{2} = {c=8}{c},
  cell{1}{10} = {c=4}{c},
  cell{1}{14} = {c=4}{c},
  cell{1}{1} = {r=3}{},
  cell{2}{2} = {c=2}{},
  cell{2}{4} = {c=2}{},
  cell{2}{6} = {c=2}{},
  cell{2}{8} = {c=2}{},
  cell{2}{10} = {c=2}{},
  cell{2}{12} = {c=2}{},
  cell{2}{14} = {c=2}{},
  cell{2}{16} = {c=2}{},
  colsep=3.5pt,
  stretch=0
}
\hline
\textbf{Method}  & Mixed-Modal    &                &                  &                &                     &                &                 &                & Cross-Modal  &              &                 &              & Uni-Modal                  &              &              &              \\ \hline
 & \textbf{Mix}   &                & \textbf{Mix-Cam} &                & \textbf{Mix-Cam-ID} &                & \textbf{Mix-ID} &                & \textbf{All} &              & \textbf{Indoor} &              & \textbf{I$\rightarrow$I} &              & \textbf{V$\rightarrow$V} &              \\ \hline
& \textbf{R1}    & \textbf{mAP}   & \textbf{R1}      & \textbf{mAP}   & \textbf{R1}         & \textbf{mAP}   & \textbf{R1}     & \textbf{mAP}   & \textbf{R1}  & \textbf{mAP} & \textbf{R1}     & \textbf{mAP} & \textbf{R1}                & \textbf{mAP} & \textbf{R1}  & \textbf{mAP} \\ \hline
SAAI\cite{SAAI}  & 96.01 & 74.59 & 90.63 & 72.51 & 84.30 & 65.94 & 52.49 & 53.30 & 73.87 & 69.71 & 84.19 & 82.59 & 89.29 & 93.06 & 98.24 & 93.46\\
SAAI\cite{SAAI} +\NM  & 97.27 & 80.47 & 92.66 & 76.81 & 88.51 & 73.24 & 66.23 & 62.45 & 74.25 & 71.08 & - & - & 92.11 & 94.57 & 98.60 & 94.08 \\ \hline
Upper-bound & 95.74 & 78.20 & 91.18 & 74.94 & 85.44 & 68.71 & 59.38 & 55.70 & 73.87 & 69.71 & 84.19 & 82.59 &88.06 & 92.80 & 98.85 & 95.20 \\ \hline
\end{tblr}
}
\vspace{-0.3cm}
\caption{Performance of SAAI\cite{SAAI} VI-ReID technique in mixed, cross, and uni-modal settings on the SYSU-MM01 dataset compared to upper-bound. The upper-bound is a model that contains three separate SAAI models for V,I, and VI images. }
\label{tab:SYSU_Upper}
\end{table*}

\subsubsection{The influence of hyperparameters.}
In this section, we present a bar chart (Fig. \ref{fig:lambda}) to examine the detailed influence of hyperparameters by gradually increasing their
value. As we can see, the best performance is achieved when $\lambda_1$ is set to 0.4, $\lambda_2$ is set to 0.6, and $\lambda_3$ is set to 0.4, respectively. The upward trend of the bars demonstrates the effectiveness of each loss. 
\begin{figure}[!h]
    \centering
    \vspace{-0.2cm}
    \includegraphics[width=0.99\linewidth]{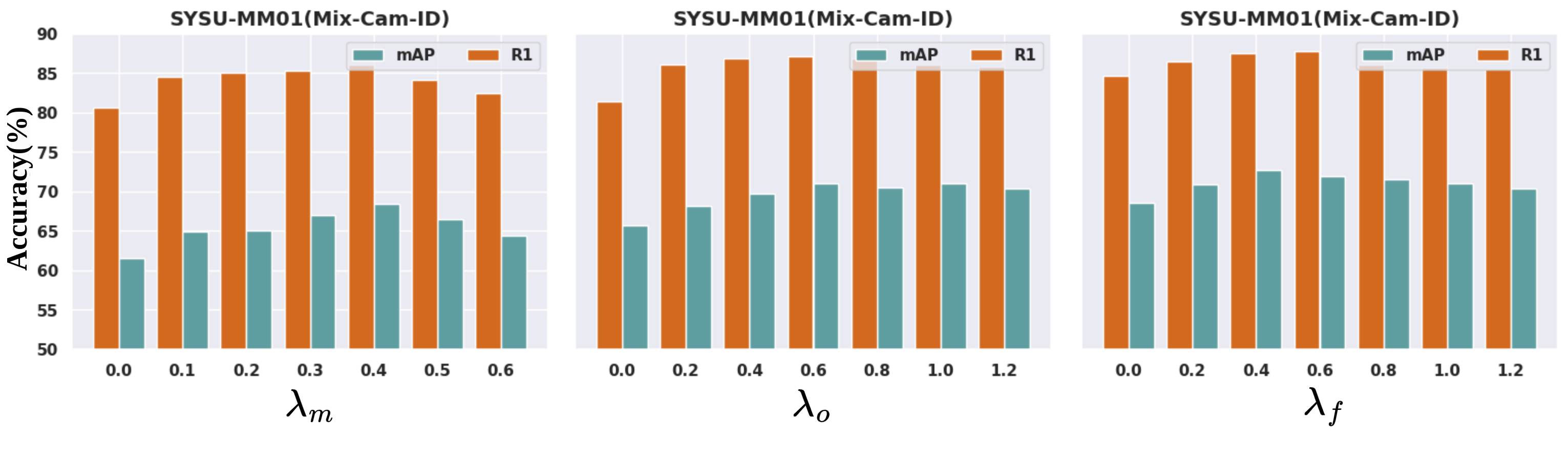}
    \vspace{-0.54cm}
    \caption{Influence of different $\lambda_m$,$\lambda_o$ and $\lambda_f$ values on SYSU-MM01 in Mix-Cam-ID evaluation.}
    \label{fig:lambda}
    \vspace{-0.35cm}
\end{figure}

\subsubsection{Choice of Backbone.}
We test the performance of our method on two main choices of backbones, ResNet \cite{he2016deep}, and ViT \cite{alexey2020image}. For vision transformer models, we resized images to 224 by 224 pixels, and we used the extracted feature from the last layer as representative features. We observe that despite its success in recent times \cite{islam2022recent}, standard ViT models struggle to perform on par with ResNet. This is likely due to the large image resolution ViTs ($224 \times 224$) are designed to train on. This size allows patches to be as large as $16 \times 16$. However, our input images were originally resized to a shape $288 \times 144$ restricts us from using a ViT architecture reliably as the ViT architecture dimensions depend on input dimensions. In this regard, ResNet is a much more flexible backbone that can work on a wide range of image resolutions and therefore delivers better performance. All models were trained following the standard implementation settings as described in Section 4.1 of the main paper, with the exception of ViT inputs being resized to $224 \times 224$, instead of $288 \times 144$.

\begin{table}[!htbp]
\centering
\resizebox{0.5\columnwidth}{!}{%
\begin{tabular}{|c|c|cc|cc|}
\hline
\multirow{2}{*}{\textbf{Backbone}} &
  \multirow{2}{*}{\textbf{\# Parameters}} &
  \multicolumn{2}{c|}{\textbf{Cross-Modal}} &
  \multicolumn{2}{c|}{\textbf{Mix-Cam-ID}} \\ \cline{3-6} 
  &
   &
  \multicolumn{1}{c|}{\textit{\textbf{R1}}} &
  \textit{\textbf{mAP}} &
   \multicolumn{1}{c|}{\textit{\textbf{R1}}} &
  \textit{\textbf{mAP}} \\ \hline
ResNet-18 & 29.02M & \multicolumn{1}{c|}{67.68}      &  63.51     & \multicolumn{1}{c|}{80.49} & 62.61 \\
 ResNet-50 & 58.15M & \multicolumn{1}{c|}{73.43} & 70.92 & \multicolumn{1}{c|}{87.56 } & 72.0 \\
 ViT-B-16  & 102.87M & \multicolumn{1}{c|}{55.47} & 54.97  & \multicolumn{1}{c|}{72.23} & 57.51 \\
 ViT-L-16  & 331.55M & \multicolumn{1}{c|}{52.57} & 53.04 & \multicolumn{1}{c|}{68.14} & 57.18 \\ \hline
 \end{tabular}%
 }
 \caption{The influence of the choice of baseline backbone on the performance of the proposed method.}
 \label{tab:backbone}
 \end{table}

\subsection{Feature Distributions Visualization}
We visualized the feature distributions of the baseline and our modules using UMAP \cite{mcinnes2018umap} in the 2D space, as shown in Fig. \ref{fig:umap}(a-d). The results indicate that the modality-erased feature brings embeddings of the same person closer across modalities compared to the baseline (see circular features in Fig. \ref{fig:umap}(a,b)). Meanwhile, the modality-related component pushes apart intra-modality features of different individuals. Together, \NM effectively leverages both components to better separate identities and reduce modality discrepancy within the mixed-modal gallery.

\begin{figure}[h!]
    \centering
    \begin{subfigure}{.23\linewidth}
      \centering
      \frame{\includegraphics[width=\linewidth, height=0.67\linewidth]{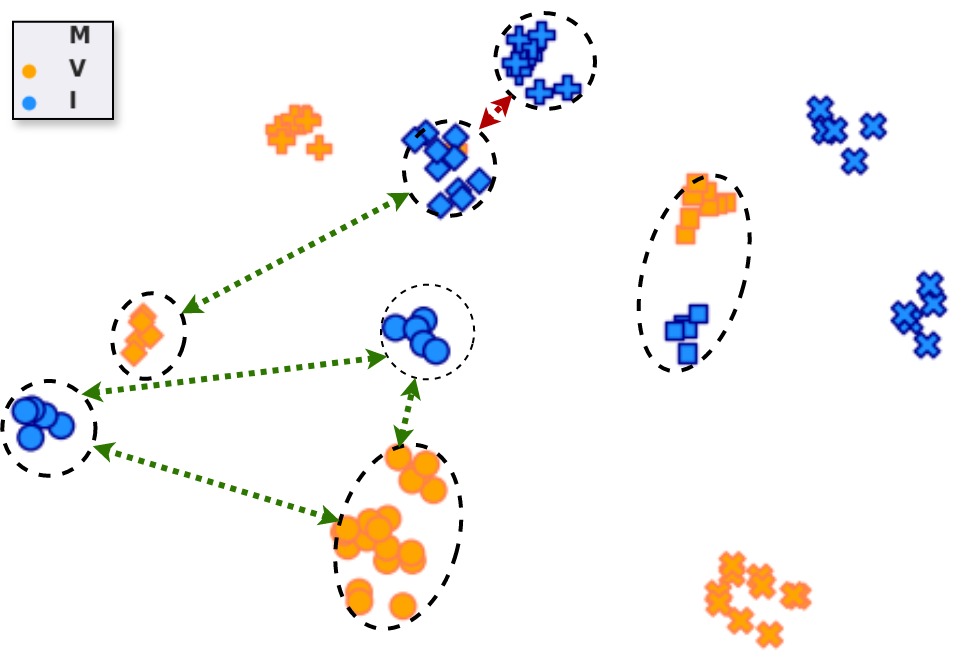}}
      \caption{Baseline}
    \end{subfigure} %
    \hspace{0.2cm}
    \begin{subfigure}{.23\linewidth}
      \centering
      \frame{\includegraphics[width=\linewidth, height=0.67\linewidth]{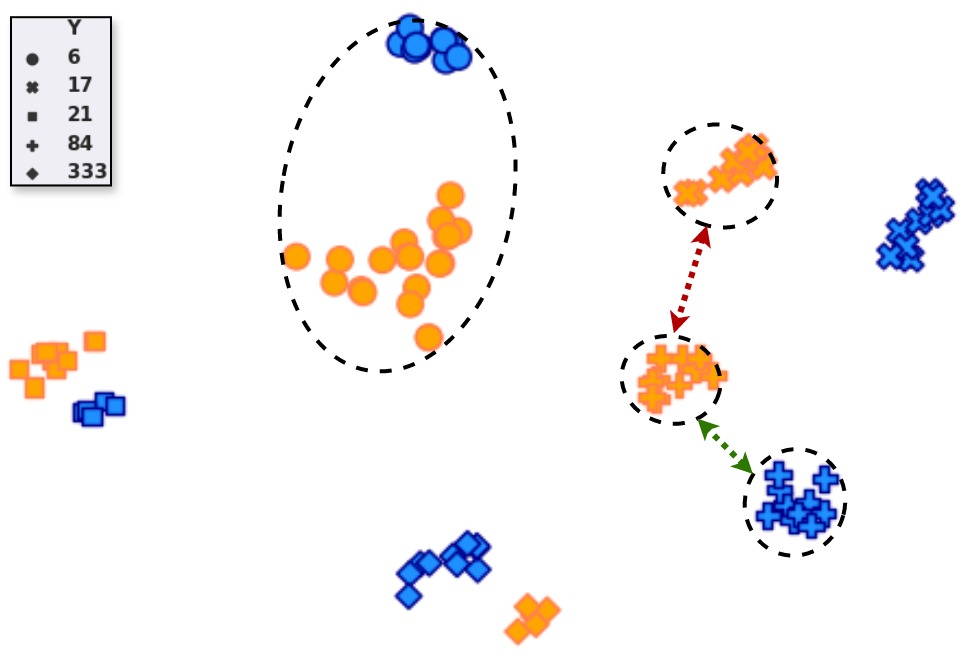}}
      \caption{Modality-Erased}
    \end{subfigure}
    \begin{subfigure}{.23\linewidth}
      \centering
      \frame{\includegraphics[width=\linewidth, height=0.67\linewidth]{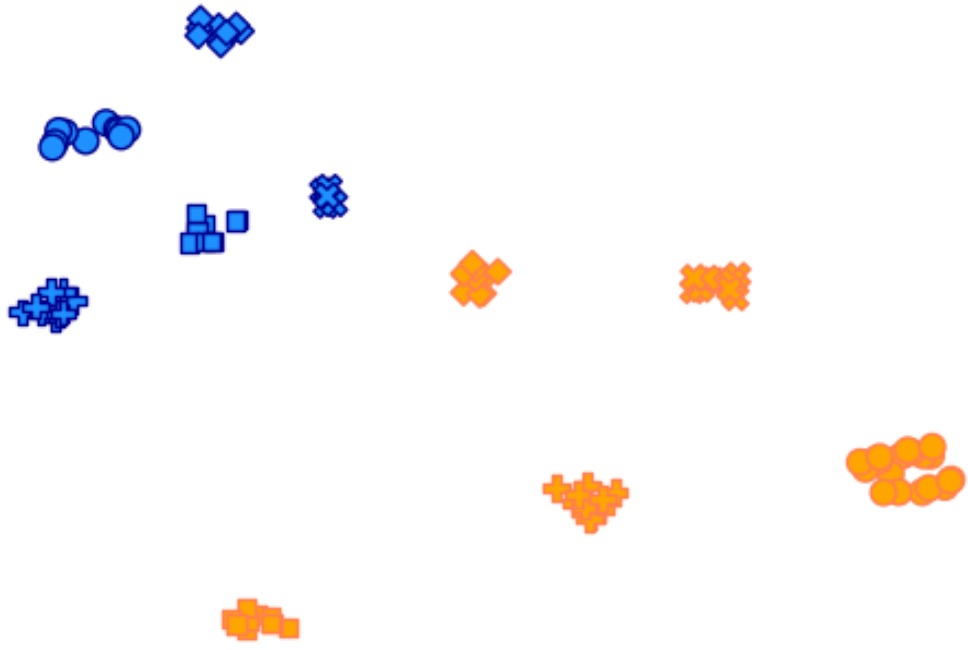}}
      \caption{Modality-Related }
    \end{subfigure}%
    \hspace{0.2cm}
    \begin{subfigure}{.23\linewidth}
      \centering
      \frame{\includegraphics[width=\linewidth, height=0.67\linewidth]{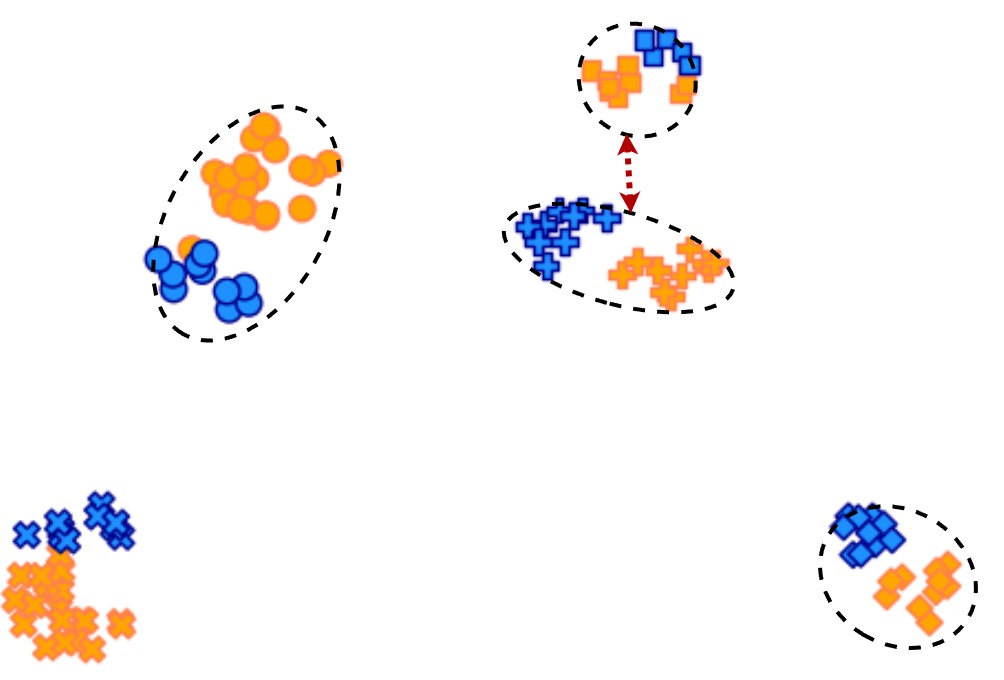}}
      \caption{\NM}
    \end{subfigure}%
    \vspace{-0.35cm}
    \caption{The distribution of feature embeddings in the 2D feature space, where orange and blue colors denote the V and I. The samples with the same shape are from the same person. The green and red arrows show the distance between the same person's features and different persons, respectively.    }
    \label{fig:umap}
\end{figure}

\clearpage
\pagebreak


\end{document}